\documentclass{article} 
\usepackage{iclr2026_conference,times}

\usepackage{hyperref}
\usepackage{url}
\usepackage{smile}
\usepackage{enumitem}
\usepackage{wrapfig}
\usepackage{lipsum}
\usepackage{colortbl}
\usepackage{xcolor}
\definecolor{waterblue}{cmyk}{0.50, 0, 0.20, 0} 

\newcommand{\piref}{\pi_{\text{ref}}}
\newcommand{\epsilonRM}{\epsilon_{\mathrm{RM}}}
\newcommand{\epsilonopt}{\epsilon_{\text{opt}}}
\title{Best-of-Majority: Minimax-Optimal \\
Strategy for Pass@$k$ Inference Scaling}

\author{Qiwei Di\thanks{Equal contribution} , Kaixuan Ji$^*$, Xuheng Li$^*$, Heyang Zhao, Quanquan Gu\\
Department of Computer Science\\
University of California, Los Angeles\\
CA 90095, USA \\
\texttt{\{qiwei2000, kaixuanji, xuheng.li, hyzhao, qgu\}@cs.ucla.edu}
}

\ifdefined\final
\usepackage[disable]{todonotes}
\else
\usepackage[textsize=tiny]{todonotes}
\fi
\usepackage[compact]{titlesec}
\usepackage[compact]{titlesec}
\titlespacing*{\section}
{0pt}{0.5ex}{0.5ex}
\titlespacing*{\subsection}
{0pt}{0.5ex}{0.5ex}
\titlespacing*{\subsubsection}
{0pt}{0.5ex}{0.5ex}

\allowdisplaybreaks

\iclrfinalcopy 
\begin{document}

\maketitle

\begin{abstract}

LLM inference often generates a batch of candidates for a prompt and selects one via strategies like majority voting or Best-of-$N$ (BoN). For difficult tasks, this single-shot selection often underperforms. Consequently, evaluations commonly report Pass@$k$: the agent may submit up to $k$ responses, and only the best of them is used when computing regret.
Motivated by this, we study inference scaling in the more general Pass@$k$ inference setting, and prove that neither majority voting nor BoN exhibits the desirable scaling with $k$ and the sampling budget $N$.
Combining the advantages of majority voting and BoN, we propose a new inference strategy called Best-of-Majority (BoM), with a pivotal step that restricts the candidates to the responses with high frequency in the $N$ samples before selecting the top-$k$ rewards.
We prove that when the sampling budget is $N=\tilde\Omega(C^*)$, the regret of BoM is $O(\epsilon_{\mathrm{opt}}+\sqrt{\epsilon_{\mathrm{RM}}^2C^*/k})$, where $C^*$ is the coverage coefficient, $\epsilon_{\mathrm{RM}}$ is the estimation error of the reward model, and $\epsilon_{\mathrm{opt}}$ is the estimation error of reward at the optimal response.
We further establish a matching lower bound, certifying that our algorithm is minimax optimal. Beyond optimality, BoM has a key advantage: unlike majority voting and BoN, its performance does not degrade when increasing $N$.
Experimental results of inference on math problems show BoM outperforming both majority voting and BoN.

\end{abstract}

\section{Introduction}
Scaling law serves as a powerful tool for guiding the \textit{training} of large language models (LLMs), providing insight into how increased training compute, data, and model size contribute to performance improvements.
Originating in the early days of deep neural networks \citep{hestness2017deep,rosenfeld2019constructive}, the concept has since demonstrated remarkable predictive power across a variety of domains, including strategic board games \citep{jones2021scaling}, image generation \citep{henighan2020scaling,yu2022scaling,peebles2023scalable}, video modeling \citep{brooks2024video}, language generation \citep{kaplan2020scaling,hoffmann2022training,achiam2023gpt}, retrieval systems \citep{fang2024scaling,cai2025exploring}, and reward modeling \citep{gao2023scaling,rafailov2024scaling}. While training-time scaling has proven effective, it is also highly resource-intensive. As a result, increasing attention has been directed toward a complementary paradigm: \textit{inference}, which examines how model performance can be improved after training.
This relationship between additional compute at inference time and performance improvement is known as the inference scaling law \citep{brown2024large,snell2024scaling,wu2024inference,guo2025deepseek}.

\begin{table}[ht]
\label{table:1}
\centering
\caption{Comparison of Pass@$k$ inference strategies. Our algorithm BoM is the first minimax-optimal Pass@$k$ inference strategy. Compared with majority voting and BoN, BoM is scaling-monotonic, indicating that the optimal performance can be achieved with large sampling budget $N$, making it preferable when scaling up $N$ to achieve better performance. Additionally, the term $O(\sqrt{\epsilonRM^2C^*/k})$ in the regret of BoM scales optimally with $k$, while majority voting suffers from constant regret. BoN lacks the regret upper bound in the Pass@$k$ inference problem.}
\begin{tabular}{cccc}
\toprule
Algorithm & Worst-case regret & Scaling-monotonic & Optimal $k$-scaling\\
\midrule
Majority voting  & \makecell{$\Omega(1)$\\{\footnotesize (Theorem \ref{thm: lowerWM})}} & No & No\\
Best-of-$N$ & \makecell{$\Omega(\min\{1, \sqrt{\epsilonRM^2N/k}\})$\\{\footnotesize (Theorem \ref{thm: lowerBoN}}} & No &Unknown\\
\rowcolor{waterblue!50!} Best-of-Majority (Ours)  & ~~~\makecell{$O(\epsilonopt+\sqrt{\epsilonRM^2C^*/k})$\\{\footnotesize (Theorem \ref{theorem:upper_bound})}}~~~ & Yes & Yes\\
\midrule
Lower Bound & \makecell{$\Omega(\epsilonopt+\sqrt{\epsilonRM^2C^*/k})$\\{\footnotesize (Theorem \ref{theorem:lower_bound_general})}} & - & -\\
\bottomrule
\end{tabular}
\end{table}
Compared to training-time scaling, inference scaling allows for increasing computational cost in several distinct ways, including expanding the generation input via chain-of-thought prompting \citep{wei2022chain,li2024chain}, incorporating iterative self-improvement, \citep{zheng2023judging,wu2024meta}, and applying search-based algorithms \citep{yao2023tree,feng2023alphazero,gao2024interpretable,zhang2024rest}. It can also be realized through repeated sampling, using strategies such as majority voting \citep{wang2022self,lewkowycz2022solving,li2023making} or Best-of-$N$ (BoN) \citep{lightman2023let}. In parallel, a growing line of works has sought to establish theoretical guarantees for inference strategies. \citet{wu2024inference} provided convergence bounds and rates for the scaling of majority voting algorithms. \citet{huang2024self} showed that BoN can achieve self-improvement via a special mechanism called sharpening. \citet{huang2025best} analyzed the sample complexity of BoN and proposed a pessimistic inference algorithm with provable benefits. \citet{chen2025towards} bridged test-time computation with theoretical analysis of transformers in the setting of in-context learning of linear regression.

While most existing analyses focus on inference algorithms that output a single response, there are tasks that allow for multiple candidate outputs, where it is considered solved if any one of them is correct. This setting is captured by the Pass@$k$ metric \citep{li2022competition}. Building on this metric, we propose a novel \textbf{Pass@$k$} inference framework, in which the inference algorithm is allowed to generate $N$ responses and return up to $k$ of them.
Since $N > k$, the performance depends not only on generating a diverse set of candidates but also on the algorithm’s ability to effectively select the $k$ outputs that are most likely to be correct.
\citet{brown2024large} conducted empirical studies on this inference framework and observed the relationship between the coverage and the performance of the algorithm. However, this work is restricted to the majority voting and BoN inference strategies, and failed to theoretically justify the inference scaling law.

As there have been few works on understanding the scaling of the Pass@$k$ inference problem, we are motivated to investigate the following fundamental question:

\emph{Q1: What is the optimal scaling of the Pass@$k$ inference problem?}

To answer this question, we derive a minimax lower bound as a function of $k$ that characterizes the fundamental limits of any Pass@$k$ inference strategy, establishing the theoretical scaling behavior for Pass@$k$ inference problems.

Going one step further, we also aim to evaluate existing inference strategies for the Pass@$k$ inference problem and find a strategy that achieves the optimal scaling. Beyond standard metrics like regret and sample complexity, we further introduce a formal definition of \textit{scaling-monotonicity} \citep{huang2025best}, which captures whether an inference algorithm maintains (or improves) its performance as the number of samples $N$ increases. This leads to our second question:

\emph{Q2: What inference strategies are scaling-monotonic and optimal in the Pass@$k$ inference setting?}

Unfortunately, our analysis reveals that majority voting and BoN are not scaling-monotonic. Furthermore, these methods face fundamental limitations that make it difficult, if not impossible, to attain the optimal regret scaling with respect to $k$. To address this issue, we propose a new inference strategy, Best-of-Majority (BoM), which integrates the core ideas of both majority voting and BoN. We establish a regret upper bound for BoM that matches the minimax lower bound, thereby demonstrating that our algorithm is minimax optimal. Please refer to Table \ref{table:1} for detailed results.

We summarize our main contributions as follows:
\begin{itemize}[leftmargin=*]
    \item \textbf{Inference scaling laws for Pass@$k$.} We show that the minimax lower bound of the regret is $\Omega(\epsilonopt+\sqrt{\epsilonRM^2C^*/k})$ for any Pass@$k$ inference strategy, where $\epsilonopt$ is the error of the reward model at the optimal response, $\epsilonRM$ is the expected error of the reward model, and $C^*$ is the coverage of the reference LLM.
    \item \textbf{Optimal algorithm for Pass@$k$.} We propose a new Pass@$k$ inference strategy called Best-of-Majority (BoM). At the core of BoM is a step similar to majority voting that restricts the candidates to the responses with high frequencies in the generated samples, before selecting responses with top-$k$ rewards. We prove that the regret of BoM is $O(\epsilonopt+\sqrt{\epsilonRM^2C^*/k})$ with sample complexity $N=\tilde\Theta(C^*)$, thus matching the minimax lower bound without increasing the computation overhead. With a formal definition of scaling monotonicity, we show that BoM is scaling monotonic, while majority voting and BoN are not.
    \item \textbf{Experiments.} We compare our algorithm BoM against majority voting and BoN. Our results empirically demonstrate the superiority of BoM against majority voting and BoN and verify the scaling monotonic properties of three algorithms, which corroborates our theoretical results.
\end{itemize}

\noindent\textbf{Notations.}
We use $[M]$ to denote the set of integers $\{1, 2, \dots, M\}$.
We use $\ind[\cdot]$ to denote the indicator function.
We use $\delta_{ij}$ to denote the Kronecker delta, i.e., $\delta_{ij}=1$ if $i=j$, and $\delta_{ij}=0$ otherwise.
We use $y, y_i$ to denote the elements in the set of response $\cY$, $\hat y,\hat y_i$ to denote the generated responses, and $\tilde y, \tilde y_i$ to denote the final outputs.
We use standard asymptotic notations $O(\cdot)$, $\Omega(\cdot)$, and $\Theta(\cdot)$, and use $\tilde O(\cdot)$, $\tilde \Omega(\cdot)$ and $\tilde \Theta(\cdot)$ to further hide the logarithmic factors. 

\section{Related Work}

\noindent\textbf{Inference-time scaling.} Compared to training-time scaling laws, the study of inference-time scaling laws has emerged much more recently. \citet{sardana2024beyond} extended the Chinchilla scaling law \citep{hoffmann2022training} to incorporate inference costs. \citet{wu2024inference} conducted a systematic study of inference scaling laws, analyzing a range of inference strategies including greedy search, majority voting, best-of-$N$, weighted voting, and two variants of tree-based search algorithms. Concurrently, \citet{snell2024scaling} analyzed the inference scaling problem by searching against process-based verifier reward models. In contrast, \citet{brown2024large} explored repeated sampling as a simple scaling method to improve performance. \citet{chen2024more} studied the performance of majority voting and a variant that incorporates a filtering mechanism. They observed that as the number of generated samples $N$ increases, performance initially improves but eventually declines.  They also proposed a predictive scaling model to characterize the performance trend. \citet{muennighoff2025s1} developed simple methods to construct a sample-efficient test-time scaling dataset. Moreover, \citet{huang2025sample} studied the sample complexity and representation ability of test-time scaling with transformers.

\noindent\textbf{Inference strategies.} One of the most straightforward inference strategies is best-of-$N$, which has been widely adopted in the inference of language models \citep{stiennon2020learning,nakano2021webgpt,touvron2023llama,gao2023scaling}. For its theoretical guarantees, \citet{yang2024asymptotics} established a connection between the asymptotic behavior of BoN and KL-constrained reinforcement learning methods, characterizing this relationship through information-theoretic quantities.
\citet{beirami2024theoretical} provided a tighter upper bound for the KL divergence between the BoN policy and the reference policy. \citet{mroueh2024information} proved guarantees for BoN algorithm from a information theoretic view. \citet{huang2025best} further provided guarantees
 on performance when the estimated reward model and
 true reward are mismatched. \citet{aminian2025best} extended the analysis to a smoothed variant of BoN called SBoN. {They proved that BoN can suffer from reward overoptimization, and SBoN can serve as a mitigation.} Another common inference strategy is majority voting \citep{lewkowycz2022solving,wang2022self,li2023making}. \citet{wu2024inference} established convergence bounds and rates characterizing how the performance of majority voting algorithms scales with the number of samples. Other inference strategies include variants of BoN \citep{jinnai2024regularized,qiu2024treebon}, rejection sampling \citep{liu2023statistical,xu2024genarm}, and search-based algorithms \citep{yao2023tree,feng2023alphazero,gao2024interpretable,zhang2024rest}.

\noindent\textbf{Pass@$k$ alignment.} To the best of our knowledge, the theoretical Pass@$k$ inference framework is novel and remains unexplored in the existing literature. {Moreover, aligning the training process with different inference algorithms has also emerged as a promising direction \citep{balashankar2024infalign}. In this direction}, Pass@$k$ has also been proved useful in the training of large language models. \citet{tang2025optimizing} demonstrated that training language models using a Pass@$k$-based objective can lead to improved overall model performance. More recently, \citet{chen2025pass} used Pass@$k$ as the reward to train the language model and observe improvements on its exploration ability. \citet{liang2025beyond} proposed training methods to mitigate entropy collapse, which in turn lead to improved performance on the Pass@$k$ metric.

\section{Pass@$k$ Inference scaling problem}
Let $\cX$ be the set of prompts and $\cY$ the set of responses. We represent an LLM as a conditional policy $\pi(\cdot\mid x)$ that maps each prompt $x\in\cX$ to a distribution over $\cY$. We have access to a reference policy $\pi_{\text{ref}}$, which, for instance, can be trained using the supervised finetuning (SFT) method. For each pair $(x, y) \in \cX \times \cY$, we assume the existence of a ground-truth reward model $r^*: \cX \times \cY \to [0,1]$, which evaluates the quality of response $y$ given prompt $x$.

During inference time, we can use the reference policy $\piref$ to generate multiple responses. To evaluate the quality of these responses, we utilize an imperfect reward model $\hat r: \cX \times \cY \rightarrow [0,1]$, which provides approximate assessments of response quality. For a given prompt $x$, we make the following assumptions regarding the accuracy of the reward model.
\begin{assumption}
[Reward Estimation Error]
\label{assump:epsilon}
The expected squared error between $r^*$ and $\hat r$ is upper bounded by $\epsilonRM^2(x)$, i.e,  
\begin{align*}
    \EE_{y \sim \piref(\cdot|x)}\Big[\big(r^*(x,y)-\hat r(x,y)\big)^2\Big] \le \epsilon_{\text{RM}}^2(x).
\end{align*}
\end{assumption}
{This assumption is the standard squared loss of the reward model, with respect to the reference policy $\piref$. The same assumption has been made in prior work like \citet{huang2025best}}.

{We further assume that the optimal answer is unique, which is a natural condition in many domains where the correctness of the final solution is verifiable, such as in mathematical problems.}
\begin{assumption}
\label{assump: reward}
    There exists a unique $y^* = \argmax_{y \in \cY}{r^*(x,y)}$, with $r^*(x,y^*)=1$. Moreover, the estimated reward at $y^*$ is close to optimal, satisfying 
    \begin{align*}
        |r^*(x,y^*) - \hat r(x,y^*)| = \epsilonopt(x).
    \end{align*}
    Combining Assumption \ref{assump:epsilon} with Assumption~\ref{assump: reward}, we directly know $\piref(y^*|x) \cdot \epsilon^2_{\text{opt}}(x) \le \epsilonRM^2(x)$.
\end{assumption} 
In practice, an accurate reward model is crucial for the post-training and inference of large language models. A common approach is to align the model with human preference data through supervised learning or reinforcement learning from human feedback (RLHF)~\citep{ouyang2022training,casperopen,zhu2024starling,yang2024regularizing}. Since the training of the reward model extensively studied and is not the focus of this work, we directly assume access to a pre-training reward model that satisfies Assumptions~\ref{assump:epsilon} and \ref{assump: reward}.

In this work, we study a novel setting called the \textbf{Pass@$k$} inference scaling problem.
Different from the settings where the model is allowed to generate and submit $k$ candidate responses, our goal is to maximize the highest ground-truth reward of the $k$ samples. Specifically, for a given prompt $x$, the model is allowed to generate up to $N$ candidate responses and select a subset ${y_1, y_2, \ldots, y_k}$ for submission. Increasing $N$ improves the likelihood of obtaining high-quality outputs, but also incurs greater computational cost, a trade-off between accuracy and efficiency. We consider the following regret metric:
\begin{align}
\label{eq:regret}
    \text{Regret}(x) =  \EE_{\pi^*}\big[r^*(x,\cdot)\big] - \EE_{y_1,y_2,\ldots,y_k}\Big[\max_{1\le i \le k} \{r^*(x,y_i)\}\Big],
\end{align}
where $\pi^* = \pi^*(\cdot|x)$ is the maximizer of $r^*$. 

In tasks with a unique correct answer, such as mathematical problem solving, the ground-truth reward model $r^*$ functions as a binary verifier, returning values in $\{0,1\}$. In this case, the regret \eqref{eq:regret} naturally aligns with the Pass@$k$ metric \citep{li2022competition}, since minimizing~\eqref{eq:regret} is equivalent to maximizing the probability that at least one of the $k$ selected responses is correct.
\begin{remark}
    Compared with the sample-and-evaluate framework \citep{huang2025best}, our framework goes one step further by explicitly characterizing the dependence on $k$. This dependence constitutes a novel focus of our analysis, as it has not been examined in prior works on inference-time algorithms \citep{huang2024self,huang2025best,verdun2025soft}.
\end{remark}
{In the setting of test-time (inference) scaling, responses are generated by $\piref$. Sampling from a good reference model can lead to performance improvement. Motivated by this, we need a metric to evaluate the performance of the reference model.} Following \citet{huang2025best}, we introduce the reference policy's $L_1$-coverage coefficient as follows:
\begin{align}
\label{eq: coverage}
    C^*(x) := \EE_{y \sim \pi^*(\cdot|x)} \big[{\pi^*(y|x)}/{\piref(y|x)}\big].
\end{align}
Moreover, the uniform coverage coefficient is defined as
\begin{align}
\label{eq: uniform coverage}
    C^*_{\infty}(x) := \sup_y \big[{\pi^*(y|x)}/{\piref(y|x)}\big].
\end{align}
Since Assumption~\ref{assump: reward} ensures that the optimal policy $\pi^*$ is deterministic and uniquely defined as $\pi^*(y|x) = \ind(y=y^*)$, the $L_1$ and uniform coverage coefficients coincide. Consequently, we have $C^*(x) = C^*_{\infty}(x)= 1/\piref(y^*|x)$.

Besides the regret, we are also concerned with the following important property of the algorithm, named as \textit{scaling-monotonicity} \citep{huang2025best}.
We provide the formal definition as follows:
\begin{definition}
\label{def:consistent}
Assume that $k$, prompt $x$ and the coverage coefficient $C^*(x)$ are fixed. An algorithm is \textit{scaling-monotonic} if for any $\delta>0$, there exists $\epsilon_0>0$ and $N_0\in\NN_+$ such that for any $N\ge N_0$ and any instance that satisfies Assumption \ref{assump:epsilon} with $\epsilonRM(x)\le\epsilon_0$, the regret satisfies
\begin{align*}
    \text{Regret}(x) \le \delta.
\end{align*}
\end{definition}
Intuitively, a scaling-monotonic algorithm should achieve arbitrarily small regret if the reward model $\hat r$ is accurate and sufficiently many samples of responses from $\piref$ are observed. Furthermore, scaling monotonicity also guarantees that the performance of the algorithm does not degrade when increasing $N$. Therefore, it is a crucial property in practice because the sampling budget $N$ can be easily scaled up in hard instances instead of requiring accurate tuning. {While this concept is not completely new, as far as we know, we are the first to formally define this property.}

\section{Suboptimality of Existing inference strategies}
\label{sec:suboptimal}

In this section, we first introduce two commonly used strategies for LLM inference, namely (weighted) majority voting (Section \ref{sec:MV}) and Best-of-$N$ (BoN, Section \ref{sec:BON}). We will show that neither strategy is scaling-monotonic by constructing hard instances where the inference strategies suffer from constant regret even when $N\to\infty$. Additionally, the Pass@$k$ inference problem is less stringent than Pass@$1$, since it only requires success in any of the $k$ sampled attempts rather than a single one. Consequently, the regret is expected to decrease as $k$ increases, suggesting a negative association between regret and the sampling budget $k$.


\subsection{(Weighted) majority voting}
\label{sec:MV}
\begin{wrapfigure}{r}{0.62\textwidth}
\begin{minipage}{0.62\textwidth}
\vspace{-5pt}
\begin{algorithm}[H]
\caption{(Weighted) Majority Voting}
\begin{algorithmic}[1]\label{alg:WM}
    \REQUIRE Reference policy $\piref$, sampling budget $N$, number of candidates $k$, (estimated reward model $\hat r$, weight function $w(\cdot)$).
    \STATE Observe context $x$.
    \STATE Independently generate $N$ responses $\hat \cY = \{\hat y_1,\hat y_2,\ldots, \hat y_N\}$ from $\piref(\cdot|x)$.
    \IF{$|\hat \cY| \le k$} 
    \RETURN $\hat \cY$.
    \ELSE
    \STATE Calculate frequency of each response $y\in\hat\cY$: 
     $\hat \pi(y) = \frac{1}{N} \sum_{i=1}^N \ind[\hat y_i=y]$.
    \IF{\textit{weighted}}
    \STATE Query reward labels $(\hat r(x, \hat y_1), \dots, \hat r(x, \hat y_N))$
    \STATE Select $\tilde y_1,\ldots, \tilde y_k = \text{Top-k} \big\{y \in \hat \cY: w(\hat r(y))\cdot\hat \pi(y)\big\}$.
    \ELSE \STATE 
    Select $\tilde y_1,\ldots, \tilde y_k = \text{Top-k} \big\{y \in \hat \cY: \hat \pi(y)\big\}$.
    \ENDIF
    \RETURN $\{\tilde y_1,\ldots, \tilde y_k\}$.
    \ENDIF
\end{algorithmic}
\end{algorithm}
\end{minipage}
\end{wrapfigure}
Majority voting is a simple ensemble method for LLM inference: Multiple responses to the same prompt are sampled using the reference policy $\piref(\cdot|x)$ to make the responses diverse enough, and the answer occurring most often is selected as the final output.

Specifically, let $\hat y_1,\ldots, \hat y_N$ denote the $N$ generated responses for a given query. After calculating the frequency of each response $\hat \pi(y) = \frac{1}{N} \sum_{i=1}^N \ind(\hat y_i=y)$, the final prediction is then chosen as the answer that appears most frequently among these samples, i.e.,
\begin{align*}
    \tilde y_1,\ldots, \tilde y_k = \text{Top-k} \big\{y \in \hat \cY: \hat \pi(y)\big\}.
\end{align*}
Majority voting has demonstrated strong empirical performance \citep{wang2022self,lewkowycz2022solving,li2023making}. With a reliable reward model $\hat r$, it can be further enhanced by weighting candidate frequencies with reward scores. Using an increasing weighting function $w(\cdot)$, the selection rule becomes:
\begin{align*}
    &\tilde y_1,\ldots, \tilde y_k 
    = \text{Top-k} \big\{y \in \hat \cY: w\big(\hat r(y)\big) \cdot \hat \pi(y)\big\}.
\end{align*}
While the reward weighting introduces extra computation for reward evaluation, weighted majority voting has been shown to achieve better performance than the unweighted version \citep{wu2024inference}.
Despite its empirical success, we show that (weighted) majority voting is suboptimal in the worst case, even when the exact reward function is available, i.e., $\epsilonRM^2(x) = 0$.
\begin{theorem}\label{thm: lowerWM}
    For the (weighted) majority voting Algorithm~\ref{alg:WM} with weight function $w(\cdot)$, assume that $C^*(x) \ge 1 + 2kw(1)/w(1/2)$. Then, there exists an instance $\cI=(\cX, \cY, \pi^*, r^*, \piref, \hat r)$ such that the coverage coefficient is $C^*(x)$, and $\hat r = r^*$ satisfies Assumptions \ref{assump:epsilon} and \ref{assump: reward} with $\epsilonRM(x) = \epsilonopt(x) = 0$.
If $N \ge 9C^*(x)\log(2k+2)$, the algorithm suffers from a constant regret:
\begin{align*}
\text{Regret}(x)=\Omega\big(1\big).
\end{align*}
\end{theorem}
Majority voting relies on exploiting the reference model’s distribution. Consequently, the hard case can be constructed by designing multiple distinct ``bad'' answers, each receiving higher probability under $\piref$ {than the probability of the optimal answer $\pi_{\text{ref}}(y^*)$}.
    Theorem \ref{thm: lowerWM} demonstrates that increasing the sampling budget $N$ or the number of submitted responses $k$ does not guarantee consistent improvement for (weighted) majority voting. In fact, when $N$ is sufficiently large, (weighted) majority voting incurs constant regret even if the reward model is accurate.
\subsection{Best-of-N}
\label{sec:BON}
\begin{wrapfigure}{r}{0.53\textwidth}
\vspace{-22pt}
\begin{minipage}{0.53\textwidth}
\begin{algorithm}[H]
\caption{Best-of-$N$ (BoN)}
\begin{algorithmic}[1]\label{alg:BoNk}
    \REQUIRE Estimated reward model $\hat r$, reference policy $\piref$, sampling budget $N$, number of candidates $k$.
    \STATE Observe context $x$.
    \STATE Independently generate $N$ responses $\hat \cY = \{\hat y_1,\hat y_2,\ldots, \hat y_N\}$ from $\piref(\cdot|x)$.
    \STATE Query reward labels $(\hat r(x,y_1), \ldots, \hat r(x,y_N))$.
    \IF{$|\hat \cY| \le k$} 
    \RETURN $\hat \cY$.
    \ELSE 
    \STATE Select $\tilde y_1,\ldots, \tilde y_k = \text{Top-k} \big\{y \in \hat \cY: \hat r(x, y)\big\}$.
    \RETURN $\{\tilde y_1,\ldots, \tilde y_k\}$.
    \ENDIF
\end{algorithmic}
\end{algorithm}
\end{minipage}
\end{wrapfigure}
Best-of-$N$ is another effective LLM inference strategy. Instead of aggregating answers by frequency, the model generates multiple candidate responses for the same query and then selects the single best response according to a reward model $\hat r$. Formally, given $N$ sampled responses $\hat y_1,\ldots, \hat y_N$, the Best-of-$N$ strategy selects the outputs that maximize the reward signal $\hat r$, i.e.,
\begin{align*}
    \tilde y_1,\ldots, \tilde y_k = \text{Top-k} \big\{y \in \hat \cY: \hat r(y)\big\}.
\end{align*}
For the BoN algorithm, we have the following theorem on the lower bound of the regret.
\begin{theorem}\label{thm: lowerBoN}
    For BoN (Algorithm~\ref{alg:BoNk}), assume that $C^*(x) \ge 2k$. Then, there exists an instance $\cI=(\cX, \cY, \pi^*, r^*, \piref, \hat r)$ such that the coverage coefficient is $C^*(x)$, and $(\hat r,r^*)$ satisfies Assumptions~\ref{assump:epsilon} and \ref{assump: reward} with $\epsilonRM(x)$ and $\epsilonopt(x)$. If $N \le C^*(x)$, Algorithm \ref{alg:BoNk} suffer from a constant regret, i.e., 
    \begin{align*}
    \text{Regret}(x)=\Omega\big(1\big).
    \end{align*}
    Otherwise, the regret satisfies
\begin{align*}
    \text{Regret}(x) = \Omega \Big( \min\Big\{1, \sqrt{N\epsilonRM^2(x)/k}\Big\}\Big).
\end{align*}
\end{theorem}
BoN leverages the reward model’s signal, but this makes it vulnerable to reward overoptimization \citep{gao2023scaling,stroebl2024inference} when the reward model is inaccurate. Thus, we construct the hard case by introducing multiple distinct ``bad” answers that are assigned higher estimated rewards. With a carefully chosen, problem-dependent sampling budget $N = \tilde\Theta(C^*(x))$, the lower bound will become $\tilde\Omega(\sqrt{C^*(x)\epsilonRM^2(x)/k})$, which aligns with the general lower bound for inference algorithms (as will be discussed in Section~\ref{sec:lower}). However, this lower bound implies that BoN is not scaling-monotonic, as for fixed $k$ and $\epsilonRM(x)$, the regret converges to a non-zero constant when $N$ becomes sufficiently large. Thus, increasing $N$ for BoN not only causes higher computational overhead, but can also degrade performance when the reward model is inaccurate.
\begin{remark}
{
For the case where $k=1$, we establish a lower bound of $\Omega(\min\{1,\sqrt{N\epsilonRM^2(x)}\})$, which recovers Theorem 3.2 in \citet{huang2025best}. Notably, \citet{huang2025best} also provided a matching upper bound when $k=1$. However, for $k > 1$, our analysis reveals an additional factor of $1/\sqrt{k}$ in the lower bound, which has not been considered in prior works. A direct application of the upper bound analysis from \citet{huang2025best} fails to capture this dependence on $k$, leaving a theoretical gap of $1/\sqrt{k}$. It remains unclear whether this suboptimality is due to limitations in the analysis or is an inherent weakness of BoN. We conjecture that obtaining a regret upper bound for BoN with the optimal $1/\sqrt{k}$ scaling under the Pass@$k$ setting may be fundamentally infeasible. We leave this to future work. 
}
\end{remark}
\section{Optimal algorithm for Pass@k inference}
\label{sec:BoM}
In Section \ref{sec:suboptimal}, we have proved that neither (weighted) majority voting nor BoN is scaling monotonic, and neither demonstrates the desirable scaling with $k$ for the Pass@$k$ inference scaling problem. {We also explicitly describe how the corresponding hard instance is constructed. By contrast, instances that do not satisfy the properties of these hard cases are precisely those on which the algorithms perform well. Thus,} our earlier analysis also reveals complementary strengths of these methods: majority voting performs well when the reference policy assigns a higher probability to the ground-truth answer than to incorrect ones, while Best-of-$N$ can be highly effective when the reward model $\hat r$ is accurate. However, each method also exhibits weaknesses, as they fail to fully exploit the available information from either the policy or the reward model. To address these limitations, we introduce a new algorithm, Best-of-Majority (BoM), which integrates the advantages of both approaches.
\begin{wrapfigure}{r}{0.55\textwidth}
\begin{minipage}{0.55\textwidth}
\begin{algorithm}[H]
\caption{Best-of-Majority (BoM)}
\begin{algorithmic}[1]\label{alg:1}
    \REQUIRE Estimated reward model $\hat r$, reference policy $\piref$, frequency threshold $\alpha$, sampling budget $N$, number of candidates $k$.
    \STATE Observe context $x$.
    \STATE Independently generate $N$ responses $\hat \cY = \{\hat y_1,\hat y_2,\ldots, \hat y_N\}$ from $\piref(\cdot|x)$.
    \STATE Calculate frequency of each response $y\in\cY$: $\hat \pi(y) = \frac{1}{N} \sum_{i=1}^N \ind(\hat y_i=y)$.
    \STATE Eliminate responses with frequency less than $\alpha$:
        $\hat \cY_\alpha = \{y \in \hat\cY: \hat \pi(y) \ge\alpha\}$.
    \STATE Query reward labels $(\hat r(x, \hat y_1), \dots, \hat r(x, \hat y_N))$.\label{alg:line5}
    \IF{$|\hat \cY_\alpha| \le k$} 
    \RETURN $\hat \cY_\alpha$.
    \ELSE \STATE Select $\tilde y_1,\ldots, \tilde y_k = \text{Top-k} \big\{y \in \hat \cY_\alpha: \hat r(y)\big\}$.
    \RETURN $\{\tilde y_1,\ldots, \tilde y_k\}$.
    \ENDIF
\end{algorithmic}
\end{algorithm}
\end{minipage}
\end{wrapfigure}
Our algorithm is built upon the principles of pessimism commonly used in reinforcement learning \citep{buckman2020importance,jin2021pessimism}. When the reference policy $\piref$ assigns low probability to a response, that response is rarely observed in the training data. Consequently, the reward model receives limited supervision in this region, leading to higher uncertainty and likelihood of error. The pessimism principle advocates making conservative predictions under such uncertainty, which motivates our design choice: we rely on the reward model only when $\piref$ assigns sufficiently high probability to the candidate. Since $\piref$ cannot be directly observed, we approximate it using empirical frequencies of generated responses. Specifically, let $\hat y_1,\ldots, \hat y_N$ denote the $N$ generated responses for a given query. We first calculate the empirical frequency of each emerging response:
\begin{align*}
    \hat \pi(y) = \frac{1}{N} \sum_{i=1}^N \ind(\hat y_i=y).
\end{align*}
Guided by the pessimism principle, we discard responses whose frequency falls below a threshold $\alpha$, retaining only the subset
\begin{align*}
    \hat \cY_\alpha = \{y \in \hat\cY: \hat \pi(y) \ge\alpha\}.
\end{align*}
Then we query the reward model on the surviving candidates and select the top $k$ responses according to their predicted rewards, $\tilde y_1,\ldots, \tilde y_k = \text{Top-k} \big\{y \in \hat \cY_\alpha: \hat r(y)\big\}$. The following theorem demonstrates the upper bound of BoM.
\begin{theorem}\label{theorem:upper_bound}
Assume that the threshold is $\alpha=3/(4C^*(x))$. Then the regret of BoM (Algorithm~\ref{alg:1}) satisfies
\begin{align*}
    \mathrm{Regret}(x) \le \epsilonopt(x) + O\Big(\sqrt{{C^*(x)\epsilonRM^2(x)}/{k}}\Big) {+ O\big(C^*(x)e^{-N/(32C^*(x))}\big)}.
\end{align*}
{Moreover, when the sampling budget satisfies $N\ge 16 C^*(x) \log \big(kC^*(x)/\epsilonRM^2(x)\big)$, we have
\begin{align*}
    \mathrm{Regret}(x) \le \epsilonopt(x) + \tilde O\Big(\sqrt{{C^*(x)\epsilonRM^2(x)}/{k}}\Big).
\end{align*}
}
\end{theorem}
Here, we let $N\ge 16 C^*(x) \log \big(kC^*(x)/\epsilonRM^2(x)\big)$ with a term dependent on $\epsilonRM(x)$ to simplify the results. In general, we can replace the $\epsilonRM(x)$ with any accuracy $\epsilon$ that we wish to achieve. Thus, Our selection of $N$ do not require the knowledge of $\epsilonRM(x)$, and our results will continue to work when $\epsilonRM(x)=0$. When $\epsilon_{\mathrm{opt}}(x) \ll \sqrt{C^*(x)\epsilon_{\mathrm{RM}}^2(x)}$, the second term dominates, and consequently the overall regret scales as $1/\sqrt{k}$, consistent with the intuition that increasing $k$ enlarges the candidate set and thereby makes the problem easier. Moreover, for fixed $x$, $k$, and $C^*(x)$, the regret bound converges to $0$ as $N \to \infty$ and $\epsilon_{\mathrm{RM}}(x) \to 0$. This yields the following corollary.
\begin{corollary}
    BoM (Algorithm \ref{alg:1}) is scaling-monotonic.
\end{corollary}
\textbf{Computational Complexity.} According to Theorem \ref{theorem:upper_bound}, the BoM algorithm requires approximately $\tilde\Omega(C^*(x))$ samples to achieve low regret. In comparison, Theorem 3.4 in Huang et al. (2025) shows that when $k = 1$, the Best-of-$N$ (BoN) algorithm also requires $\tilde \Theta(C^*(x))$ samples. This means for Pass@$k$ inference, BoM achieves a better regret bound with a $1/\sqrt k$ improvement without incurring additional generation cost. Moreover, BoM only queries the reward model for a filtered subset of candidates (see Algorithm \ref{alg:1}, Line \ref{alg:line5}), which can reduce the number of reward evaluations.

\begin{proof}[Proof Sketch of Theorem \ref{theorem:upper_bound}]

The crucial step of BoM involves the construction of $\hat\cY_{\alpha}$ to approximate the set of all responses $y$ with $\piref(y|x)\ge\alpha$, denoted by $\cY_\alpha$. The following two properties of $\cY_\alpha$ makes it preferable as the set of candidates: Firstly, if $\tilde y_i\in\cY_{\alpha}(x)$ for all $i\in[k]$, we have an upper bound of the minimum estimation error $\min_{i\in[k]}\Delta_i$, where $\Delta_i=|\hat r(x, \tilde y_i)-r^*(x, \tilde y_i)|$:
\begin{align}\label{eq:Delta_i_bound}
\min_{i\in[k]}\Delta_i\le\sqrt{\sum\nolimits_{i=1}^k\Delta_i^2/k}\le\sqrt{\sum\nolimits_{i=1}^k\piref(\tilde y_i|x)\Delta_i^2/(\alpha k)}\le\sqrt{\epsilonRM^2(x)/(\alpha k)},
\end{align}
where we used the property $\piref(\tilde y_i|x)\ge\alpha$ in the second inequality. Secondly, since $\piref(y^*|x)\ge1/C^*(x)$, we have $y^*\in\cY_{1/C^*(x)}$. Therefore, if $\hat\cY_{\alpha}(x)=\cY_{1/C^*(x)}(x)$, the algorithm either outputs $y^*$ among the $k$ submitted responses, incurring zero regret, or outputs $k$ responses with $\hat r(x, \tilde y_i)\ge \hat r(x, y^*)$, where the regret can be decomposed as
\begin{align*}
&r^*(x, y^*)-r^*(x, \tilde y_i)\le\underbrace{|r^*(x, y^*)-\hat r(x, y^*)|}_{\epsilonopt(x)}+\underbrace{[\hat r(x, y^*)-\hat r(x, \tilde y_i)]}_{\le0}+\underbrace{|\hat r(x, \tilde y_i)-r^*(x, \tilde y_i)|}_{\Delta_i}.
\end{align*}
We take the minimum, plug in \eqref{eq:Delta_i_bound}, and obtain
\begin{align*}
r^*(x, y^*)-\max_{i\in[k]}r^*(x, \tilde y_i)\le\epsilonopt(x)+\min_{i\in[k]}\Delta_{\tilde y_i}\le\epsilonopt(x)+\sqrt{4C^*(x)\epsilonRM^2(x)/k}.
\end{align*}
However, without direct access to $\piref$, we use the empirical frequency $\hat \pi$ instead of $\piref$ in the construction of $\hat \cY_{\alpha}$, making $\hat\cY_\alpha$ an \textit{approximation} of $\cY_\alpha$. To extend the two properties of $\cY_{\alpha}$ to $\hat\cY_{\alpha}$, we require the following event that sandwiches $\hat\cY_{3/(4C^*(x))}(x)$ with $\cY_{1/C^*(x)}(x)$ and $\cY_{1/(4C^*(x))}(x)$:
\begin{align*}
\cE: \cY_{1/C^*(x)}(x)\subset\hat\cY_{3/(4C^*(x))}(x)\subset\cY_{1/(4C^*(x))}(x).
\end{align*}
Under event $\cE$, $\alpha$ can be set as $1/(4C^*(x))$ in \eqref{eq:Delta_i_bound}. The complete expectation formula gives
\begin{align*}
\mathrm{Regret}(x)&=\EE\Big[r^*(x, y^*)-\max_{i\in[k]}r^*(x, \tilde y_i)\Big|\cE\Big]\cdot\PP(\cE)+\EE\Big[r^*(x, y^*)-\max_{i\in[k]}r^*(x, \tilde y_i)\Big|\neg\cE\Big]\cdot\PP(\neg\cE)\\
&\le\epsilonopt(x)+\sqrt{4C^*(x)\epsilonRM^2(x)/k}+\PP(\neg\cE),
\end{align*}
so it remains to characterize the probability of $\cE$.

The probability of $\cY_{1/C^*(x)}(x)\not\subset\hat\cY_{3/(4C^*(x))}(x)$ can be characterized by first bounding $\PP(y\not\in\hat\cY_{3/(4C^*(x))}(x))$ for any $y\in\cY_{1/C^*(x)}(x)$ using the Chernoff bound, and then applying the union bound with the crucial observation of $|\cY_{1/C^*(x)}(x)|\le C^*(x)$. When characterizing $\PP(\hat\cY_{3/(4C^*(x))}(x)\not\subset\cY_{1/(4C^*(x))}(x))$, we can similarly use the Chernoff bound in $\PP(y\in\hat\cY_{3/(4C^*(x))}(x))$ for any $y\in\cY(x)\backslash\cY_{1/(4C^*(x)}(x)$, but the union bound does not hold because the cardinality of the set $\cY(x)\backslash\cY_{1/(4C^*(x)}(x)$ is unknown. To resolve this issue, we assign elements of $\cY(x)\backslash\cY_{1/(4C^*(x)}(x)$ into ``bins'' $\{G_j\}$, each with capacity $1/(2C^*(x))$, i.e., $\piref(G_j|x)\le1/(2C^*(x))$.
The smallest number of bins is no more than $4C^*(x)$ because any two bins with $\piref(G_j|x)\le1/(4C^*(x))$ can be merged. With this assignment, we can bound $\PP(G_j\cap\hat\cY_{3/(4C^*(x))}(x)\neq\varnothing)$ with the Chernoff bound, and then use the union bound with the bins, which resolves the problem because the number of bins is bounded.
\end{proof}
\section{General lower bounds}\label{sec:lower}
In this section, we establish a lower bound that highlights the fundamental factors influencing the Pass@$k$ inference problem. Specifically, the bound depends on the coverage coefficient $C^*(x)$, the reward model estimation error $\epsilonRM^2(x)$ and $\epsilonopt(x)$, and the number of candidates~$k$. It matches the upper bound in Theorem \ref{theorem:upper_bound}, which indicates that the algorithm BoM is minimax optimal. 
\begin{theorem}\label{theorem:lower_bound_general}
For a given prompt $x$, assume that $C^*(x)\ge2k$. Then for any algorithm $\cA$ for the Pass@$k$ inference problem, there exists an instance $\cI=(\cX, \cY, \pi^*, r^*, \piref, \hat r)$ such that the coverage coefficient is $C^*(x)$, and $(r^*,\hat r)$ satisfies Assumptions \ref{assump:epsilon} and \ref{assump: reward}. Moreover, and regret can be lower bounded by
\begin{align*}
\mathrm{Regret}(x)=\Omega\Big(\epsilonopt(x) + \sqrt{C^*(x)\epsilonRM^2(x)/k}\Big).
\end{align*}
\end{theorem}
Theorem \ref{theorem:lower_bound_general} shows that the term $\epsilonopt(x)$ is unavoidable in the Pass@$k$ inference problem and does not diminish as the number of candidates $k$ increases. In contrast, the component associated with the expected squared loss, $\epsilonRM(x)$, decreases at a rate of $1/\sqrt{k}$. This bound matches the upper bound for BoM (Theorem \ref{theorem:upper_bound}), demonstrating that BoM is minimax optimal.
\section{Experiments}
In this section, we empirically verify the effectiveness of our proposed BoM algorithm on mathematical reasoning tasks. 

\subsection{Experiment Setup}
\label{sec:experiment}
\noindent\textbf{Models and Datasets.} We use Qwen3-4B-Instruct-2507~\citep{qwen3technicalreport} (Qwen3-4B) as the reference policy $\piref$\footnote{Please see Appendix~\ref{app:addtional-exp} for results on additional models.}. We adopt AceMath-7B-RM~\citep{acemath2024} as the reward model $\hat r$, a mathematical reward model trained on a large corpus generated by different language models which is selected due to its strong performance and moderate size. We adopt the widely used GSM8K~\citep{cobbe2021training}, MATH-500~\citep{hendrycks2021measuring}, and AIME24\footnote{\url{https://huggingface.co/datasets/di-zhang-fdu/AIME\_1983\_2024}} dataset as our testing corpus.
We first sample $N$ trajectories and call the reward model to evaluate each trajectory.
The answers are then extracted from the trajectories and clustered by mathematical equivalence\footnote{The equivalence is determined through the standard implementation in Qwen2.5-Math repository and refers the readers to their public implementation for more details. \url{https://github.com/QwenLM/Qwen2.5-Math/blob/a45202bd16f1ec06f433442dc1152d0074773465/evaluation/grader.py##L73}}.
For each answer group, we use the average of the rewards of all the corresponding trajectories as the reward of this group.
We also calculate the frequency of each answer group as an estimation of $\piref(\cdot)$.

\noindent\textbf{Method and Baselines.} Given a specific $k$, we consider our method BoM, and two baselines, majority voting and BoN. In BoM, we set a threshold $\alpha$ and select the $k$ answers (up to mathematical equivalence) with highest reward score and frequency greater than $\alpha$. In BoN, we directly select the $k$ answers (up to mathematical equivalence) with highest rewards. As for majority voting, we directly select $k$ answers (up to mathematical equivalence) with highest frequency. {Additionally, we consider the SBoN algorithm studied in \citet{aminian2025best,verdun2025soft}. Given $N$ generated answers $\{y_i\}_{i=1}^N$, we select $k$ answers in a sequence, without replacement according to a softmax distribution $p(y_i|x) \propto \exp(\beta \hat r(x,y_i))$, where $\beta$ is a tunable parameter. We examine two variants: with and without reward calibration, where the calibration is implemented in the same way as in \citep{balashankar2024infalign}. In the calibrated case, rewards are normalized using the win rate over the reference policy $\piref$. We denote these versions SBoN(C) and SBoN, respectively.}

\subsection{Results}

\begin{figure}[ht]
\centering   
\subfigure[GSM8K]{\label{fig:gsm8k}\includegraphics[width=0.32\textwidth]{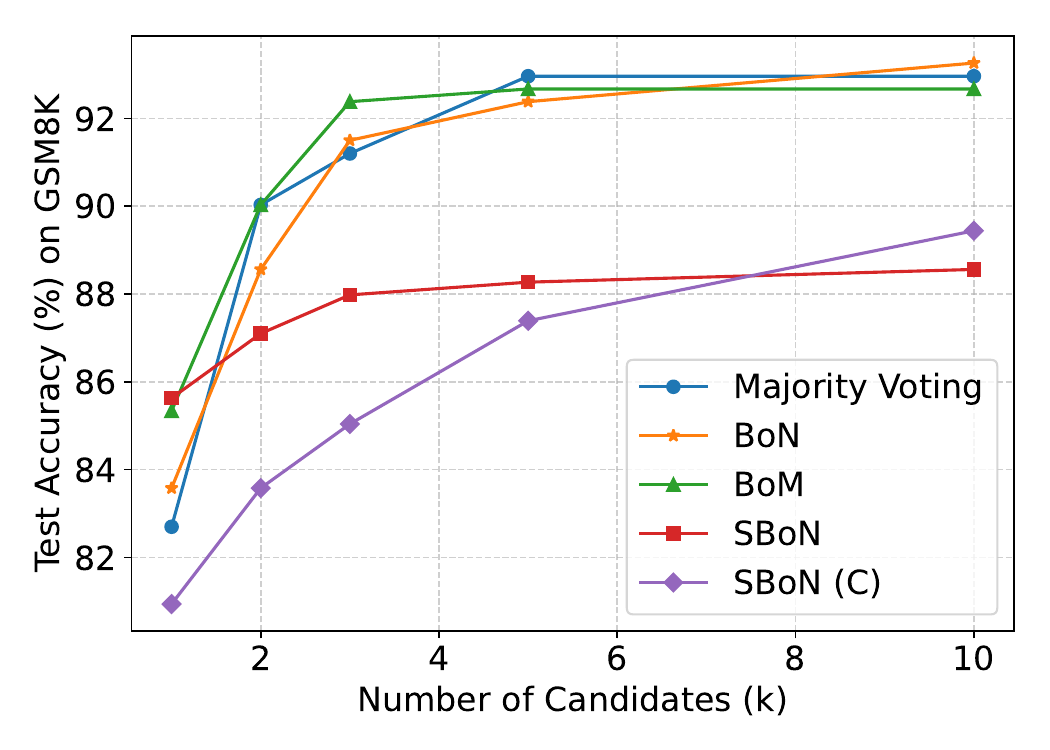}}
\subfigure[MATH-500]{\label{fig:math500}\includegraphics[width=0.32\textwidth]{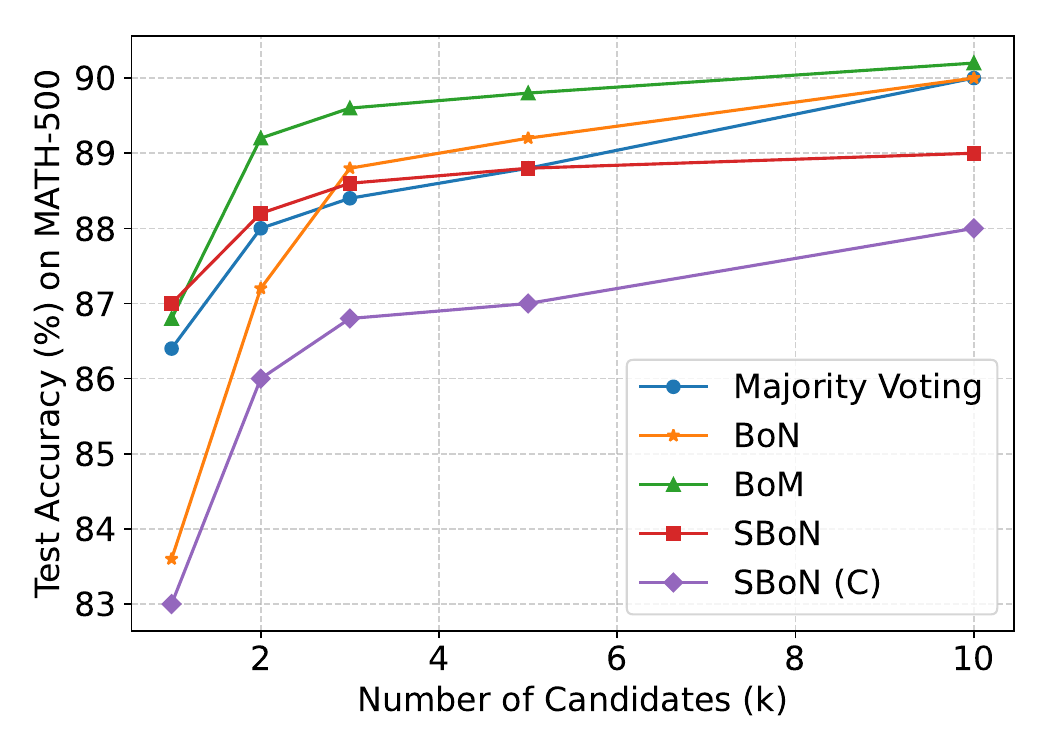}}
\subfigure[AIME24]{\label{fig:aime24}\includegraphics[width=0.32\textwidth]{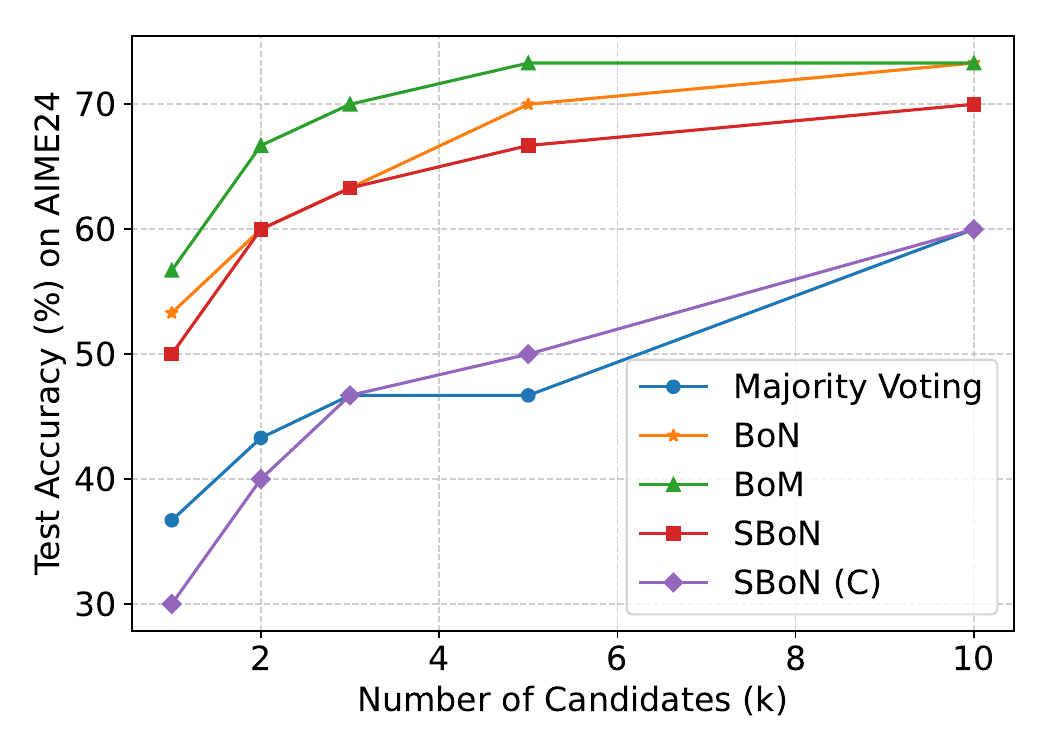}}
\caption{The results with different $k$. BoM consistently outperforms the baselines on MATH-500 for all $k$ and on AIME24, GSM8K when $k$ is small, and matches the performance of baselines in other settings.}
\label{fig:scaling-k}
\end{figure}
\begin{figure}[ht]
\centering   
\subfigure[$k=1$]{\label{fig:k=1}\includegraphics[width=0.32\textwidth]{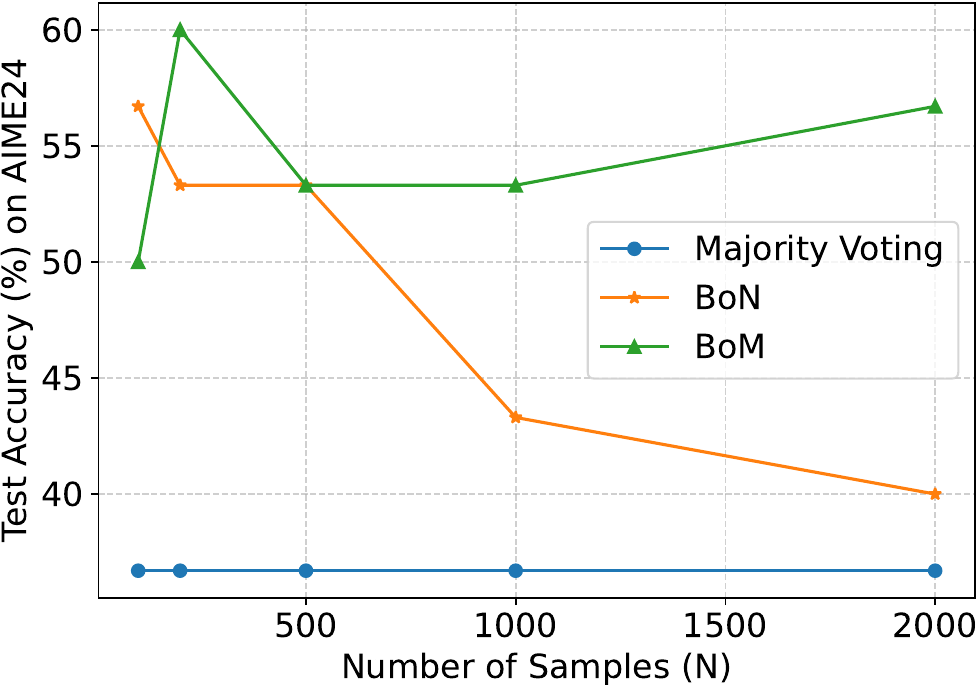}}
\subfigure[$k=3$]{\label{fig:k=3}\includegraphics[width=0.32\textwidth]{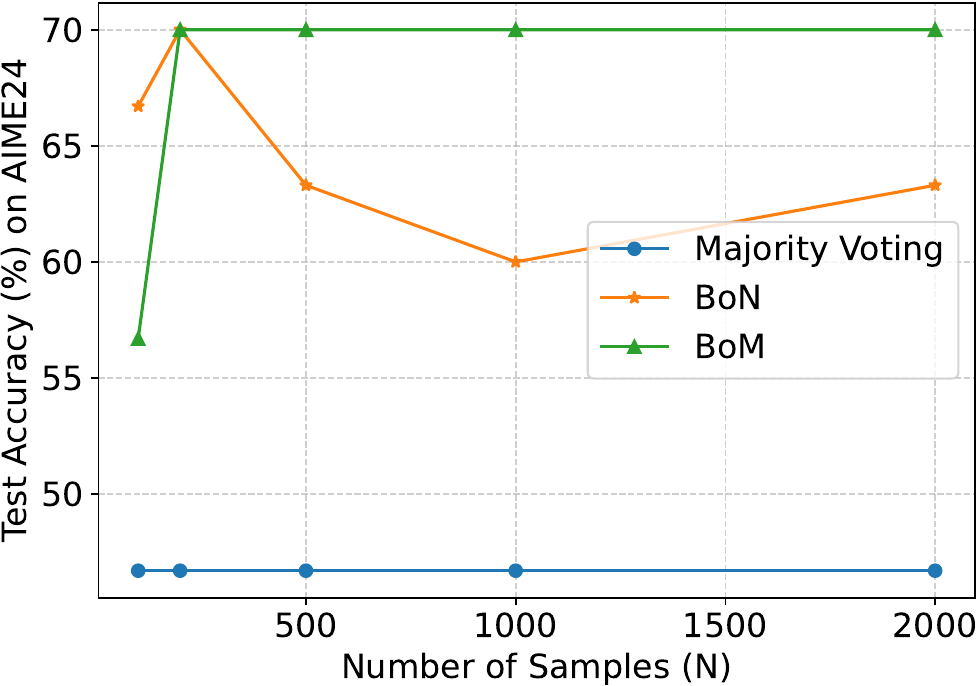}}
\subfigure[$k=5$]{\label{fig:k=5}\includegraphics[width=0.32\textwidth]{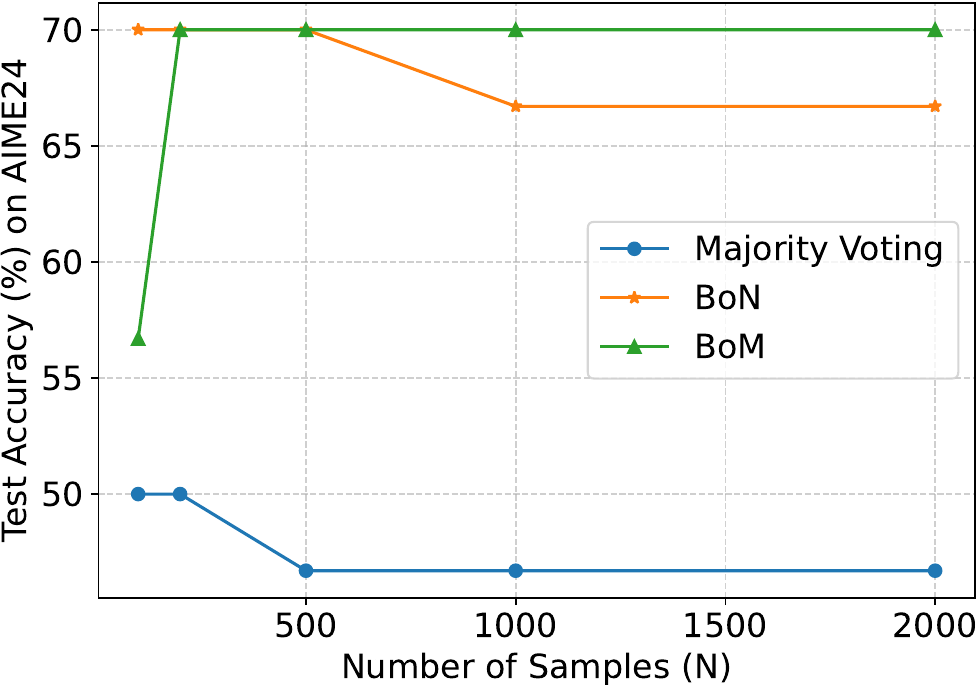}}
\caption{The results with fixed $k$ and different $N$. When $N$ increases, the performance of BoN is likely to decrease over all the $k$. The performance of Majority voting remains at a low level. Among them, BoM has a more consistent performance and outperforms baselines with larger $N$.}
\label{fig:scaling-n}
\vspace{-0.2in}
\end{figure}

\noindent\textbf{Results with varying $k$.}
\label{sec:exp results}
We first plot the results for $k\in\{1, 2, 3, 5, 10\}$ in Figure~\ref{fig:gsm8k} for GSM8K, Figure~\ref{fig:math500} for MATH-500, and Figure~\ref{fig:aime24} for AIME24. We sample $N=2000$ for GSM8k, and $N=500$ for MATH-500 and AIME24. On MATH-500, the performance of BoM consistently outperforms the baselines. On GSK8K and AIME24, BoM also shows a large improvement over majority voting and outperforms BoN for small $k$. These results empirically verify the effectiveness of the BoM algorithm.

\noindent\textbf{Results with varying $N$.} We also study the performance of the three methods under different sample sizes. We conduct the experiments on the AIME24 dataset and vary $N$ between 100 and 2000, with $k=1,3,5$. Except for the case of $N=100$ where the threshold of BoM is set to $\alpha = 0.015$, we use $\alpha = 0.005$ in all other settings. We compile the results in Figure~\ref{fig:scaling-n}.
the performance of majority voting remains consistently low, which aligns with Theorem~\ref{thm: lowerWM}, demonstrating that majority voting incurs constant regret and does not benefit from increased sample size. The performance of BoN tends to degrade as $N$ increases. In contrast, when $N \ge 200$, BoM consistently outperforms both baselines and does not decrease significantly with the increase of $N$. This observation is consistent with our theoretical results, as BoM is scaling-monotonic.
{
\subsection{Ablation study of $\alpha$}
In this section, we present an ablation study of the hyperparameter $\alpha$ in our BoM algorithm. When $\alpha=0$, BoM will degrade to BoN, meaning that as $\alpha$ approaches 0, BoM behaves similarly to BoN. We further observe that selecting a larger $\alpha$ can potentially improve the performance when $k$ is small, although this comes at the cost of deteriorated performance for larger $k$. More experimental results are included in Appendix \ref{sec:further}.
}
\begin{table}[htbp]
\centering
\scriptsize
\vspace{5pt}
\caption{{Ablation study of $\alpha$ of BoM in MATH500, Qwen3-4B}}
\label{tab:large}
\begingroup
\resizebox{0.6\textwidth}{!}{%
\begin{tabular}{c|c|c|c|c|c}
\toprule
Pass@$k$&1 & 2 &3 & 5& 10 \\
\midrule
$\alpha=0$ (BoN) 
    &  83.6     & 87.2 & 88.8& 89.2& 90 \\
$\alpha=0.003$ 
    &  86     & 88.8 & 89& 89.2& 90 \\
$\alpha=0.005$&    86.8   & 89.2  &89.6 &89.8 &90.2 \\
  $\alpha=0.007$  &    86.6   & 89 & 89.6& 89.8& 90.2\\
  $\alpha=0.011$  &    86.6   & 88.8 & 89.4& 89.4& 89.6\\
  $\alpha=0.015$  &    87.4   & 88.8 & 89.4& 89.4& 89.8\\
  Majority voting  &    86.4   & 88 & 88.4&88.8& 90\\
\bottomrule
\end{tabular}%
}
\endgroup
\end{table}
\section{Conclusion}

In this work, we demonstrate the scaling laws of the Pass@$k$ inference problem by displaying the minimax lower bound of the regret and proposing the algorithm BoM with regret matching the lower bound. We also show that BoM has the advantage of scaling monotonicity compared with majority voting and BoN, which makes BoM preferable when scaling up the generation budget. For future work, {it would be interesting to explore instance-dependent lower bounds, instead of the minimax lower bound as Theorems \ref{thm: lowerWM}, \ref{thm: lowerBoN} and \ref{theorem:lower_bound_general}. In addition, our current analysis assumes the uniqueness of the optimal response, such that the coverage coefficient will not depend on the choice of optimal policy $\pi^*$. How to define generalized coverage coefficients under the multiple optimum setting remains to be explored.}  

\section*{Acknowledgment}
We thank the anonymous reviewers and area chair for their helpful comments. QD, KJ, XL, HZ and QG are supported in part by the National Science Foundation DMS-2323113, CPS-2312094, IIS-2403400 and the Sloan Research Fellowship. QD is also supported in part by Amazon PhD Fellowship. The views and conclusions contained in this paper are those of the authors and should not be interpreted as representing any funding agencies.
\section*{Ethics statement}
Our work investigates a novel Pass@$k$ inference problem, focusing on the theoretical analysis of different inference strategies. In addition, we propose a new algorithm, Best-of-Majority (BoM), which achieves optimal theoretical guarantees, and we further provide empirical validation to support its effectiveness. Importantly, our experiments focus on solving mathematical problems with LLMs, and the language models do not generate or promote harmful content, nor does it raise issues related to discrimination, bias, or fairness.
\section*{Reproducibility statement}
In this paper, we conduct experiments with open-source LLMs on widely used mathematical datasets. A detailed description of the models and datasets is provided in Section~\ref{sec:experiment}, while the key experimental parameters are discussed in Section~\ref{sec:exp results}. On the theoretical side, we present a proof sketch of the upper bound of the BoM algorithm in Section~\ref{sec:BoM}, with the complete proof deferred to Appendix~\ref{sec:upper}. Appendix~\ref{sec:lowerbounds} contains several lower-bound results, corresponding to the theorems in Sections~\ref{sec:MV}, \ref{sec:BON}, and~\ref{sec:lower}.
\bibliography{iclr2026_conference}

@article{cobbe2021training,
  title={Training verifiers to solve math word problems},
  author={Cobbe, Karl and Kosaraju, Vineet and Bavarian, Mohammad and Chen, Mark and Jun, Heewoo and Kaiser, Lukasz and Plappert, Matthias and Tworek, Jerry and Hilton, Jacob and Nakano, Reiichiro and others},
  journal={arXiv preprint arXiv:2110.14168},
  year={2021}
}

@article{hendrycks2021measuring,
  title={Measuring mathematical problem solving with the math dataset},
  author={Hendrycks, Dan and Burns, Collin and Kadavath, Saurav and Arora, Akul and Basart, Steven and Tang, Eric and Song, Dawn and Steinhardt, Jacob},
  journal={arXiv preprint arXiv:2103.03874},
  year={2021}
}

@article{huang2025best,
  title={Is best-of-n the best of them? coverage, scaling, and optimality in inference-time alignment},
  author={Huang, Audrey and Block, Adam and Liu, Qinghua and Jiang, Nan and Krishnamurthy, Akshay and Foster, Dylan J},
  journal={arXiv preprint arXiv:2503.21878},
  year={2025}
}

@article{wang2022self,
  title={Self-consistency improves chain of thought reasoning in language models},
  author={Wang, Xuezhi and Wei, Jason and Schuurmans, Dale and Le, Quoc and Chi, Ed and Narang, Sharan and Chowdhery, Aakanksha and Zhou, Denny},
  journal={arXiv preprint arXiv:2203.11171},
  year={2022}
}

@inproceedings{li2023making,
  title={Making language models better reasoners with step-aware verifier},
  author={Li, Yifei and Lin, Zeqi and Zhang, Shizhuo and Fu, Qiang and Chen, Bei and Lou, Jian-Guang and Chen, Weizhu},
  booktitle={Proceedings of the 61st Annual Meeting of the Association for Computational Linguistics (Volume 1: Long Papers)},
  pages={5315--5333},
  year={2023}
}

@article{lewkowycz2022solving,
  title={Solving quantitative reasoning problems with language models},
  author={Lewkowycz, Aitor and Andreassen, Anders and Dohan, David and Dyer, Ethan and Michalewski, Henryk and Ramasesh, Vinay and Slone, Ambrose and Anil, Cem and Schlag, Imanol and Gutman-Solo, Theo and others},
  journal={Advances in neural information processing systems},
  volume={35},
  pages={3843--3857},
  year={2022}
}

@article{huang2024self,
  title={Self-improvement in language models: The sharpening mechanism},
  author={Huang, Audrey and Block, Adam and Foster, Dylan J and Rohatgi, Dhruv and Zhang, Cyril and Simchowitz, Max and Ash, Jordan T and Krishnamurthy, Akshay},
  journal={arXiv preprint arXiv:2412.01951},
  year={2024}
}

@article{verdun2025soft,
  title={Soft Best-of-n Sampling for Model Alignment},
  author={Verdun, Claudio Mayrink and Oesterling, Alex and Lakkaraju, Himabindu and Calmon, Flavio P},
  journal={arXiv preprint arXiv:2505.03156},
  year={2025}
}

@article{li2022competition,
  title={Competition-level code generation with alphacode},
  author={Li, Yujia and Choi, David and Chung, Junyoung and Kushman, Nate and Schrittwieser, Julian and Leblond, R{\'e}mi and Eccles, Tom and Keeling, James and Gimeno, Felix and Dal Lago, Agustin and others},
  journal={Science},
  volume={378},
  number={6624},
  pages={1092--1097},
  year={2022},
  publisher={American Association for the Advancement of Science}
}

@article{tang2025optimizing,
  title={Optimizing language models for inference time objectives using reinforcement learning},
  author={Tang, Yunhao and Zheng, Kunhao and Synnaeve, Gabriel and Munos, R{\'e}mi},
  journal={arXiv preprint arXiv:2503.19595},
  year={2025}
}

@article{chen2025pass,
  title={Pass@ k Training for Adaptively Balancing Exploration and Exploitation of Large Reasoning Models},
  author={Chen, Zhipeng and Qin, Xiaobo and Wu, Youbin and Ling, Yue and Ye, Qinghao and Zhao, Wayne Xin and Shi, Guang},
  journal={arXiv preprint arXiv:2508.10751},
  year={2025}
}

@article{liang2025beyond,
  title={Beyond Pass@ 1: Self-Play with Variational Problem Synthesis Sustains RLVR},
  author={Liang, Xiao and Li, Zhongzhi and Gong, Yeyun and Shen, Yelong and Wu, Ying Nian and Guo, Zhijiang and Chen, Weizhu},
  journal={arXiv preprint arXiv:2508.14029},
  year={2025}
}

@article{wu2024inference,
  title={Inference scaling laws: An empirical analysis of compute-optimal inference for problem-solving with language models},
  author={Wu, Yangzhen and Sun, Zhiqing and Li, Shanda and Welleck, Sean and Yang, Yiming},
  journal={arXiv preprint arXiv:2408.00724},
  year={2024}
}

@inproceedings{jin2021pessimism,
  title={Is pessimism provably efficient for offline rl?},
  author={Jin, Ying and Yang, Zhuoran and Wang, Zhaoran},
  booktitle={International conference on machine learning},
  pages={5084--5096},
  year={2021},
  organization={PMLR}
}

@article{buckman2020importance,
  title={The importance of pessimism in fixed-dataset policy optimization},
  author={Buckman, Jacob and Gelada, Carles and Bellemare, Marc G},
  journal={arXiv preprint arXiv:2009.06799},
  year={2020}
}

@article{hestness2017deep,
  title={Deep learning scaling is predictable, empirically},
  author={Hestness, Joel and Narang, Sharan and Ardalani, Newsha and Diamos, Gregory and Jun, Heewoo and Kianinejad, Hassan and Patwary, Md Mostofa Ali and Yang, Yang and Zhou, Yanqi},
  journal={arXiv preprint arXiv:1712.00409},
  year={2017}
}

@article{henighan2020scaling,
  title={Scaling laws for autoregressive generative modeling},
  author={Henighan, Tom and Kaplan, Jared and Katz, Mor and Chen, Mark and Hesse, Christopher and Jackson, Jacob and Jun, Heewoo and Brown, Tom B and Dhariwal, Prafulla and Gray, Scott and others},
  journal={arXiv preprint arXiv:2010.14701},
  year={2020}
}

@article{jones2021scaling,
  title={Scaling scaling laws with board games},
  author={Jones, Andy L},
  journal={arXiv preprint arXiv:2104.03113},
  year={2021}
}

@article{yu2022scaling,
  title={Scaling autoregressive models for content-rich text-to-image generation},
  author={Yu, Jiahui and Xu, Yuanzhong and Koh, Jing Yu and Luong, Thang and Baid, Gunjan and Wang, Zirui and Vasudevan, Vijay and Ku, Alexander and Yang, Yinfei and Ayan, Burcu Karagol and others},
  journal={arXiv preprint arXiv:2206.10789},
  volume={2},
  number={3},
  pages={5},
  year={2022}
}

@article{rosenfeld2019constructive,
  title={A constructive prediction of the generalization error across scales},
  author={Rosenfeld, Jonathan S and Rosenfeld, Amir and Belinkov, Yonatan and Shavit, Nir},
  journal={arXiv preprint arXiv:1909.12673},
  year={2019}
}

@inproceedings{peebles2023scalable,
  title={Scalable diffusion models with transformers},
  author={Peebles, William and Xie, Saining},
  booktitle={Proceedings of the IEEE/CVF international conference on computer vision},
  pages={4195--4205},
  year={2023}
}

@article{brooks2024video,
  title={Video generation models as world simulators},
  author={Brooks, Tim and Peebles, Bill and Holmes, Connor and DePue, Will and Guo, Yufei and Jing, Li and Schnurr, David and Taylor, Joe and Luhman, Troy and Luhman, Eric and others},
  journal={OpenAI Blog},
  volume={1},
  number={8},
  pages={1},
  year={2024}
}

@article{kaplan2020scaling,
  title={Scaling laws for neural language models},
  author={Kaplan, Jared and McCandlish, Sam and Henighan, Tom and Brown, Tom B and Chess, Benjamin and Child, Rewon and Gray, Scott and Radford, Alec and Wu, Jeffrey and Amodei, Dario},
  journal={arXiv preprint arXiv:2001.08361},
  year={2020}
}

@article{hoffmann2022training,
  title={Training compute-optimal large language models},
  author={Hoffmann, Jordan and Borgeaud, Sebastian and Mensch, Arthur and Buchatskaya, Elena and Cai, Trevor and Rutherford, Eliza and Casas, Diego de Las and Hendricks, Lisa Anne and Welbl, Johannes and Clark, Aidan and others},
  journal={arXiv preprint arXiv:2203.15556},
  year={2022}
}

@article{achiam2023gpt,
  title={Gpt-4 technical report},
  author={Achiam, Josh and Adler, Steven and Agarwal, Sandhini and Ahmad, Lama and Akkaya, Ilge and Aleman, Florencia Leoni and Almeida, Diogo and Altenschmidt, Janko and Altman, Sam and Anadkat, Shyamal and others},
  journal={arXiv preprint arXiv:2303.08774},
  year={2023}
}

@inproceedings{gao2023scaling,
  title={Scaling laws for reward model overoptimization},
  author={Gao, Leo and Schulman, John and Hilton, Jacob},
  booktitle={International Conference on Machine Learning},
  pages={10835--10866},
  year={2023},
  organization={PMLR}
}

@article{rafailov2024scaling,
  title={Scaling laws for reward model overoptimization in direct alignment algorithms},
  author={Rafailov, Rafael and Chittepu, Yaswanth and Park, Ryan and Sikchi, Harshit Sushil and Hejna, Joey and Knox, Brad and Finn, Chelsea and Niekum, Scott},
  journal={Advances in Neural Information Processing Systems},
  volume={37},
  pages={126207--126242},
  year={2024}
}

@inproceedings{fang2024scaling,
  title={Scaling laws for dense retrieval},
  author={Fang, Yan and Zhan, Jingtao and Ai, Qingyao and Mao, Jiaxin and Su, Weihang and Chen, Jia and Liu, Yiqun},
  booktitle={Proceedings of the 47th International ACM SIGIR Conference on Research and Development in Information Retrieval},
  pages={1339--1349},
  year={2024}
}

@inproceedings{cai2025exploring,
  title={Exploring Training and Inference Scaling Laws in Generative Retrieval},
  author={Cai, Hongru and Li, Yongqi and Yuan, Ruifeng and Wang, Wenjie and Zhang, Zhen and Li, Wenjie and Chua, Tat-Seng},
  booktitle={Proceedings of the 48th International ACM SIGIR Conference on Research and Development in Information Retrieval},
  pages={1339--1349},
  year={2025}
}

@article{brown2024large,
  title={Large language monkeys: Scaling inference compute with repeated sampling},
  author={Brown, Bradley and Juravsky, Jordan and Ehrlich, Ryan and Clark, Ronald and Le, Quoc V and R{\'e}, Christopher and Mirhoseini, Azalia},
  journal={arXiv preprint arXiv:2407.21787},
  year={2024}
}

@article{snell2024scaling,
  title={Scaling llm test-time compute optimally can be more effective than scaling model parameters},
  author={Snell, Charlie and Lee, Jaehoon and Xu, Kelvin and Kumar, Aviral},
  journal={arXiv preprint arXiv:2408.03314},
  year={2024}
}

@article{guo2025deepseek,
  title={Deepseek-r1: Incentivizing reasoning capability in llms via reinforcement learning},
  author={Guo, Daya and Yang, Dejian and Zhang, Haowei and Song, Junxiao and Zhang, Ruoyu and Xu, Runxin and Zhu, Qihao and Ma, Shirong and Wang, Peiyi and Bi, Xiao and others},
  journal={arXiv preprint arXiv:2501.12948},
  year={2025}
}

@article{wei2022chain,
  title={Chain-of-thought prompting elicits reasoning in large language models},
  author={Wei, Jason and Wang, Xuezhi and Schuurmans, Dale and Bosma, Maarten and Xia, Fei and Chi, Ed and Le, Quoc V and Zhou, Denny and others},
  journal={Advances in neural information processing systems},
  volume={35},
  pages={24824--24837},
  year={2022}
}

@article{li2024chain,
  title={Chain of thought empowers transformers to solve inherently serial problems},
  author={Li, Zhiyuan and Liu, Hong and Zhou, Denny and Ma, Tengyu},
  journal={arXiv preprint arXiv:2402.12875},
  volume={1},
  year={2024}
}

@article{zheng2023judging,
  title={Judging llm-as-a-judge with mt-bench and chatbot arena},
  author={Zheng, Lianmin and Chiang, Wei-Lin and Sheng, Ying and Zhuang, Siyuan and Wu, Zhanghao and Zhuang, Yonghao and Lin, Zi and Li, Zhuohan and Li, Dacheng and Xing, Eric and others},
  journal={Advances in neural information processing systems},
  volume={36},
  pages={46595--46623},
  year={2023}
}

@article{wu2024meta,
  title={Meta-rewarding language models: Self-improving alignment with llm-as-a-meta-judge},
  author={Wu, Tianhao and Yuan, Weizhe and Golovneva, Olga and Xu, Jing and Tian, Yuandong and Jiao, Jiantao and Weston, Jason and Sukhbaatar, Sainbayar},
  journal={arXiv preprint arXiv:2407.19594},
  year={2024}
}

@article{yao2023tree,
  title={Tree of thoughts: Deliberate problem solving with large language models},
  author={Yao, Shunyu and Yu, Dian and Zhao, Jeffrey and Shafran, Izhak and Griffiths, Tom and Cao, Yuan and Narasimhan, Karthik},
  journal={Advances in neural information processing systems},
  volume={36},
  pages={11809--11822},
  year={2023}
}

@article{gao2024interpretable,
  title={Interpretable contrastive monte carlo tree search reasoning},
  author={Gao, Zitian and Niu, Boye and He, Xuzheng and Xu, Haotian and Liu, Hongzhang and Liu, Aiwei and Hu, Xuming and Wen, Lijie},
  journal={arXiv preprint arXiv:2410.01707},
  year={2024}
}

@article{zhang2024rest,
  title={Rest-mcts*: Llm self-training via process reward guided tree search},
  author={Zhang, Dan and Zhoubian, Sining and Hu, Ziniu and Yue, Yisong and Dong, Yuxiao and Tang, Jie},
  journal={Advances in Neural Information Processing Systems},
  volume={37},
  pages={64735--64772},
  year={2024}
}

@article{feng2023alphazero,
  title={Alphazero-like tree-search can guide large language model decoding and training},
  author={Feng, Xidong and Wan, Ziyu and Wen, Muning and McAleer, Stephen Marcus and Wen, Ying and Zhang, Weinan and Wang, Jun},
  journal={arXiv preprint arXiv:2309.17179},
  year={2023}
}

@inproceedings{lightman2023let,
  title={Let's verify step by step},
  author={Lightman, Hunter and Kosaraju, Vineet and Burda, Yuri and Edwards, Harrison and Baker, Bowen and Lee, Teddy and Leike, Jan and Schulman, John and Sutskever, Ilya and Cobbe, Karl},
  booktitle={The Twelfth International Conference on Learning Representations},
  year={2023}
}

@article{aminian2025best,
  title={Best-of-N through the Smoothing Lens: KL Divergence and Regret Analysis},
  author={Aminian, Gholamali and Shenfeld, Idan and Asadi, Amir R and Beirami, Ahmad and Mroueh, Youssef},
  journal={arXiv preprint arXiv:2507.05913},
  year={2025}
}

@inproceedings{sardana2024beyond,
  title={Beyond Chinchilla-Optimal: Accounting for Inference in Language Model Scaling Laws},
  author={Sardana, Nikhil and Portes, Jacob and Doubov, Sasha and Frankle, Jonathan},
  booktitle={International Conference on Machine Learning},
  pages={43445--43460},
  year={2024},
  organization={PMLR}
}

@article{muennighoff2025s1,
  title={s1: Simple test-time scaling},
  author={Muennighoff, Niklas and Yang, Zitong and Shi, Weijia and Li, Xiang Lisa and Fei-Fei, Li and Hajishirzi, Hannaneh and Zettlemoyer, Luke and Liang, Percy and Cand{\`e}s, Emmanuel and Hashimoto, Tatsunori},
  journal={arXiv preprint arXiv:2501.19393},
  year={2025}
}

@inproceedings{yang2024asymptotics,
  title={Asymptotics of language model alignment},
  author={Yang, Joy Qiping and Salamatian, Salman and Sun, Ziteng and Suresh, Ananda Theertha and Beirami, Ahmad},
  booktitle={2024 IEEE International Symposium on Information Theory (ISIT)},
  pages={2027--2032},
  year={2024},
  organization={IEEE}
}

@article{stiennon2020learning,
  title={Learning to summarize with human feedback},
  author={Stiennon, Nisan and Ouyang, Long and Wu, Jeffrey and Ziegler, Daniel and Lowe, Ryan and Voss, Chelsea and Radford, Alec and Amodei, Dario and Christiano, Paul F},
  journal={Advances in neural information processing systems},
  volume={33},
  pages={3008--3021},
  year={2020}
}

@article{nakano2021webgpt,
  title={Webgpt: Browser-assisted question-answering with human feedback},
  author={Nakano, Reiichiro and Hilton, Jacob and Balaji, Suchir and Wu, Jeff and Ouyang, Long and Kim, Christina and Hesse, Christopher and Jain, Shantanu and Kosaraju, Vineet and Saunders, William and others},
  journal={arXiv preprint arXiv:2112.09332},
  year={2021}
}

@article{touvron2023llama,
  title={Llama 2: Open foundation and fine-tuned chat models},
  author={Touvron, Hugo and Martin, Louis and Stone, Kevin and Albert, Peter and Almahairi, Amjad and Babaei, Yasmine and Bashlykov, Nikolay and Batra, Soumya and Bhargava, Prajjwal and Bhosale, Shruti and others},
  journal={arXiv preprint arXiv:2307.09288},
  year={2023}
}

@article{beirami2024theoretical,
  title={Theoretical guarantees on the best-of-n alignment policy},
  author={Beirami, Ahmad and Agarwal, Alekh and Berant, Jonathan and D'Amour, Alexander and Eisenstein, Jacob and Nagpal, Chirag and Suresh, Ananda Theertha},
  journal={arXiv preprint arXiv:2401.01879},
  year={2024}
}

@article{mroueh2024information,
  title={Information theoretic guarantees for policy alignment in large language models},
  author={Mroueh, Youssef},
  journal={arXiv preprint arXiv:2406.05883},
  year={2024}
}

@article{liu2023statistical,
  title={Statistical rejection sampling improves preference optimization},
  author={Liu, Tianqi and Zhao, Yao and Joshi, Rishabh and Khalman, Misha and Saleh, Mohammad and Liu, Peter J and Liu, Jialu},
  journal={arXiv preprint arXiv:2309.06657},
  year={2023}
}

@article{xu2024genarm,
  title={GenARM: Reward guided generation with autoregressive reward model for test-time alignment},
  author={Xu, Yuancheng and Sehwag, Udari Madhushani and Koppel, Alec and Zhu, Sicheng and An, Bang and Huang, Furong and Ganesh, Sumitra},
  journal={arXiv preprint arXiv:2410.08193},
  year={2024}
}

@article{qiu2024treebon,
  title={Treebon: Enhancing inference-time alignment with speculative tree-search and best-of-n sampling},
  author={Qiu, Jiahao and Lu, Yifu and Zeng, Yifan and Guo, Jiacheng and Geng, Jiayi and Wang, Huazheng and Huang, Kaixuan and Wu, Yue and Wang, Mengdi},
  journal={arXiv preprint arXiv:2410.16033},
  year={2024}
}

@inproceedings{jinnai2024regularized,
  title={Regularized best-of-n sampling to mitigate reward hacking for language model alignment},
  author={Jinnai, Yuu and Morimura, Tetsuro and Ariu, Kaito and Abe, Kenshi},
  booktitle={ICML 2024 Workshop on Models of Human Feedback for AI Alignment},
  year={2024}
}

@article{chen2024more,
  title={Are more llm calls all you need? towards the scaling properties of compound ai systems},
  author={Chen, Lingjiao and Davis, Jared Quincy and Hanin, Boris and Bailis, Peter and Stoica, Ion and Zaharia, Matei A and Zou, James Y},
  journal={Advances in Neural Information Processing Systems},
  volume={37},
  pages={45767--45790},
  year={2024}
}

@article{acemath2024,
  title={AceMath: Advancing Frontier Math Reasoning with Post-Training and Reward Modeling},
  author={Liu, Zihan and Chen, Yang and Shoeybi, Mohammad and Catanzaro, Bryan and Ping, Wei},
  journal={arXiv preprint},
  year={2024}
}

@misc{qwen3technicalreport,
      title={Qwen3 Technical Report}, 
      author={Qwen Team},
      year={2025},
      eprint={2505.09388},
      archivePrefix={arXiv},
      primaryClass={cs.CL},
      url={https://arxiv.org/abs/2505.09388}, 
}

@article{stroebl2024inference,
  title={Inference scaling flaws: The limits of llm resampling with imperfect verifiers},
  author={Stroebl, Benedikt and Kapoor, Sayash and Narayanan, Arvind},
  journal={arXiv preprint arXiv:2411.17501},
  year={2024}
}

@inproceedings{zhu2024starling,
  title={Starling-7b: Improving helpfulness and harmlessness with rlaif},
  author={Zhu, Banghua and Frick, Evan and Wu, Tianhao and Zhu, Hanlin and Ganesan, Karthik and Chiang, Wei-Lin and Zhang, Jian and Jiao, Jiantao},
  booktitle={First Conference on Language Modeling},
  year={2024}
}

@article{ouyang2022training,
  title={Training language models to follow instructions with human feedback},
  author={Ouyang, Long and Wu, Jeffrey and Jiang, Xu and Almeida, Diogo and Wainwright, Carroll and Mishkin, Pamela and Zhang, Chong and Agarwal, Sandhini and Slama, Katarina and Ray, Alex and others},
  journal={Advances in neural information processing systems},
  volume={35},
  pages={27730--27744},
  year={2022}
}

@article{casperopen,
  title={Open Problems and Fundamental Limitations of Reinforcement Learning from Human Feedback},
  author={Casper, Stephen and Davies, Xander and Shi, Claudia and Gilbert, Thomas Krendl and Scheurer, J{\'e}r{\'e}my and Rando, Javier and Freedman, Rachel and Korbak, Tomek and Lindner, David and Freire, Pedro and others},
  journal={Transactions on Machine Learning Research},
  year={2023}
}

@article{yang2024regularizing,
  title={Regularizing hidden states enables learning generalizable reward model for llms},
  author={Yang, Rui and Ding, Ruomeng and Lin, Yong and Zhang, Huan and Zhang, Tong},
  journal={Advances in Neural Information Processing Systems},
  volume={37},
  pages={62279--62309},
  year={2024}
}

@book{lattimore2020bandit,
  title={Bandit algorithms},
  author={Lattimore, Tor and Szepesv{\'a}ri, Csaba},
  year={2020},
  publisher={Cambridge University Press}
}

@article{dubhashi1996balls,
  title={Balls and bins: A study in negative dependence},
  author={Dubhashi, Devdatt P and Ranjan, Desh},
  journal={BRICS Report Series},
  volume={3},
  number={25},
  year={1996}
}

@article{dubhashi1998concentration,
  title={Concentration of measure for the analysis of randomised algorithms},
  author={Dubhashi, D and Panconesi, Alessandro},
  journal={Draft Manuscript, http://www. brics. dk/ale/papers. html},
  year={1998}
}

@article{balashankar2024infalign,
  title={InfAlign: Inference-aware language model alignment},
  author={Balashankar, Ananth and Sun, Ziteng and Berant, Jonathan and Eisenstein, Jacob and Collins, Michael and Hutter, Adrian and Lee, Jong and Nagpal, Chirag and Prost, Flavien and Sinha, Aradhana and others},
  journal={arXiv preprint arXiv:2412.19792},
  year={2024}
}

@article{cai2024internlm2,
  title={Internlm2 technical report},
  author={Cai, Zheng and Cao, Maosong and Chen, Haojiong and Chen, Kai and Chen, Keyu and Chen, Xin and Chen, Xun and Chen, Zehui and Chen, Zhi and Chu, Pei and others},
  journal={arXiv preprint arXiv:2403.17297},
  year={2024}
}

@article{chen2025towards,
  title={Towards Theoretical Understanding of Transformer Test-Time Computing: Investigation on In-Context Linear Regression},
  author={Chen, Xingwu and Lu, Miao and Wu, Beining and Zou, Difan},
  journal={arXiv preprint arXiv:2508.07571},
  year={2025}
}

@article{huang2025sample,
  title={Sample complexity and representation ability of test-time scaling paradigms},
  author={Huang, Baihe and Li, Shanda and Wu, Tianhao and Yang, Yiming and Talwalkar, Ameet and Ramchandran, Kannan and Jordan, Michael I and Jiao, Jiantao},
  journal={arXiv preprint arXiv:2506.05295},
  year={2025}
}
\bibliographystyle{iclr2026_conference}

\appendix
{
\section{Comparison with \citet{aminian2025best}}
In \citet{aminian2025best}, the theoretical guarantees of BoN are studied under a different setting, and a variant called SBoN (Soft BoN) is also analyzed. In this section, we provide a comprehensive comparison of the assumptions and theoretical regimes considered in both works. \citet{aminian2025best} made three main assumptions.
\begin{itemize}[leftmargin=*]
    \item The reward function is bounded, i.e., $0\le r^*(\cdot) \le R_{\text{max}}$, $0\le \hat r(\cdot) \le R_{\text{max}}$.
    \item The reward estimation error $\epsilon_{\beta,r}(x) := \frac{1}{\beta} \log(\mathbb{E}_{\pi_{\text{ref}}}[\exp(\beta(r^*(x,y)-\hat r(x,y))^2)])$.
    \item Maximal reward can be achieved for the estimated reward $\hat r$.
\end{itemize} 
In comparison, the first assumption aligns with our setting with $R_{\max} = 1$. The second assumption aligns with our Assumption \ref{assump:epsilon} only when $\beta = 0$. For the third assumption, we assume the maximal value can be achieved only for the true reward model $r^*$, with a unique optimizer. And the error at the optimal point is specially considered.
\\\textbf{Results of BoN}: \citet{aminian2025best} proved a regret of 
\begin{align*}
\text{Regret}(x) =\sqrt{\epsilon_{\infty,r}(x)}\Big(\sqrt{C_{\infty,\hat r, \text{ref}}(x)}+ \sqrt{C_{\infty,r^*, \text{ref}}(x)}\Big) + c \sqrt{\log\Big(1+\frac{C_{\infty,\hat r, \text{ref}}(x)-1}{N}\Big)},
\end{align*}
for some constant $c>0$. In this result, it considers two coverage definitions dependent on $\hat r$ and $r^*$. Moreover, the reward estimation error is $\beta=\infty$, which corresponds to the case of supreme norm, instead of the squared norm. When $N\rightarrow\infty$, the convergence rate is $\sqrt{1/N}$.
\\\textit{Comparison with our results}: 
One central distinction lies in our different focuses. Our work focuses on analyzing algorithmic performance within the Pass@$k$ framework, and therefore our theoretical results explicitly account for the ability to submit $k$ different answers, which is not studied in \citet{aminian2025best}. In particular, for the theory of BoN, Theorem \ref{thm: lowerBoN} establishes a \textbf{lower bound} that depends explicitly on $k$, which serves as a parallel result to \citet{aminian2025best}.
\\As for upper bound of BoM, we prove that it has an exponential decay dependence of $N$, together with a term of $O(\sqrt{C\epsilonRM^2/k})$, which directly has $k$ dependence. Although the definition of reward error differs, thus the results can not be directly compared, we have shown a faster convergence rate of $N$ than \citet{aminian2025best}.
\\\textbf{Results of SBoN}: For SBoN, \citet{aminian2025best} proved a regret of
\begin{align*}
    \text{Regret}(x) &= \sqrt{\epsilon_{\beta,r}(x)}\Big(\sqrt{C_{\infty,\hat r, \text{ref}}(x)}+ \sqrt{C_{\infty,r^*, \text{ref}}(x)}\Big) \\
    & \qquad + c \sqrt{\log\Big(1+\frac{C_{\infty,\hat r, \text{ref}}(x)-1}{N}\Big)} + \log(C_{\infty, r^*,\text{ref}}(x))/\beta.
\end{align*}
Note that $\beta \neq 0$, as the last term will explode. Thus, the setting of \citet{aminian2025best} is never same as ours. Again, when $N\rightarrow\infty$, the convergence rate of SBoN is $\sqrt{1/N}$, compared with our exponential decay. Finally, our focus is on the Pass@$k$ setting, which has never been analyzed in \citet{aminian2025best}.
}

\section{Theoretical guarantee of BoM (Algorithm \ref{alg:1})}
\label{sec:upper}
In this section, we will prove Theorem \ref{theorem:upper_bound}, which provides the theoretical upper bound of Algorithm \ref{alg:1}.
To start with, for any $\alpha>0$, we denote
\begin{align*}
\cY_\alpha(x)=\{y\in\cA(x):\piref(y|x)\ge\alpha\},
\end{align*}
indicating the set of responses with relatively high probability for $\piref$. Using the definition of the coverage coefficient \eqref{eq: coverage}, we have $y^* \in \cY_{\alpha}(x)$ as long as $\alpha \ge 1/C^*(x)$. Next, we will build the relationship between the empirical set $\hat \cY_{\alpha}(x)$ and $\cY_\alpha(x)$.
Denote $\cE$ as the event such that
\begin{align*}
\cY_{1/C^*(x)}(x)\subset\hat\cY_{3/(4C^*(x))}\subset\cY_{1/(4C^*(x))}(x).
\end{align*}
Our proof consists of two parts: 

\textbf{Step 1: }We first show that $\cE$ holds with high probability.

\textbf{Step 2:} Provided that $\cE$ holds, since $y^*\in\cY_{1/C^*(x)}(x)$, we have $y^*\in\hat\cY_{3/(4C^*(x))}$; furthermore, since $\tilde y_i\in\hat\cY_{3/(4C^*(x))}$, we have $\tilde y_i\in\cY_{1/(4C^*(x))}(x)$, so $\piref(\tilde y_i|x)\ge1/(4C^*(x))$ for every submitted response $\tilde y_i$. We can then characterize $\Delta_i=|r^*(x, \tilde y_i)-\hat r(x, \tilde y_i)|$ using the definition of the estimation error $\epsilonRM^2$. If $y^*\in\{\tilde y_1, \dots, \tilde y_k\}$, then the regret is zero; if $y^*\not\in\{\tilde y_1, \dots, \tilde y_k\}$, then using Assumption \ref{assump: reward}, we have
\begin{align*}
r^*(x, y^*)-r^*(x, \tilde y_i)\le\underbrace{|r^*(x, y^*)-\hat r(x, y^*)|}_{\epsilonopt(x)}+\underbrace{[\hat r(x, y^*))-\hat r(x, \tilde y_i)]}_{\le 0}+\underbrace{|\hat r(x, \tilde y_i)-r^*(x, \tilde y_i)|}_{\Delta_i}.
\end{align*}
Combining these parts together, we complete the proof of Theorem \ref{theorem:upper_bound}.

We now get into the details of the proof. The following lemma states that the event of $\cE$ will occur with high probability:
\begin{lemma}\label{lemma:high_prob_event}
$\cE$ holds with probability at least $1-5C^*(x)e^{-N/(32C^*(x))}$.
\end{lemma}
\begin{proof}
The proof consists of two parts that characterize the probabilities of $\cY_{1/C(x)}(x)\not\subset\hat\cY_{3/(4C^*(x))}$ and $\hat\cY_{3/(4C^*(x))}\not\subset\cA_{1/(4C^*(x))}(x)$, respectively:

\noindent\textbf{Part I: Probability of $\cY_{1/C^*(x)}(x)\not\subset\hat\cY_{3/(4C^*(x))}$.}
We first fix any $y\in\cY_{1/C^*(x)}(x)$. By Chernoff bound, we have
\begin{align}\label{eq:high_prob_pihat}
\PP\big(\hat\pi(y)<3/(4C^*(x))\big)\le\exp\Big(-\frac{N\piref(y|x)}{2}\Big(1-\frac{3}{4C^*(x)\piref(a|x)}\Big)^2\Big)\le e^{-N/(32C^*(x))},
\end{align}
where the first inequality holds due to the Chernoff bound, and the second inequality holds because $\piref(y|x)\ge1/C^*(x)$. Applying the union bound to all $y\in\cY_{1/C^*(x)}(x)$, we have
\begin{align}
\PP\big(\cY_{1/C^*(x)}(x)\not\subset\hat\cY_{3/(4C^*(x))}\big)&=\PP\bigg(\bigvee_{y\in\cY_{1/C^*(x)}(x)}\ind[\hat\pi(y)\le3/(4C^*(x))]\bigg)\nonumber\\
&\le\sum_{y\in\cY_{1/C^*(x)}(x)}\PP\big(\hat\pi(y)<3/(4C^*(x))\big)\nonumber\\
&\le1-\big|\cY_{1/C^*(x)}(x)\big|\cdot e^{-N/(32C^*(x))}\nonumber\\
&\le1-C^*(x)e^{-N/(32C^*(x))},\label{eq:high_prob_1}
\end{align}
where the first inequality holds due to the union bound, the second inequality holds due to \eqref{eq:high_prob_pihat}, and the last inequality holds because $|\cY_{1/C^*(x)}(x)|\le C^*(x)$.

\noindent\textbf{Part II: Probability of $\hat\cY_{3/(4C^*(x))}\not\subset\cA_{1/(4C^*(x))}(x)$.}
We cannot use the same union bound as \eqref{eq:high_prob_1} because the cardinality of the set to take union bound $\cY\backslash\cY_{1/(4C^*(x))}(x)$ is unknown. To resolve this issue, we first partition $\cY\backslash\cY_{1/(4C^*(x))}(x)$ into groups, then apply Chernoff bound to each group, and finally apply the union bound to the groups. This technique resolves the problem because the number of groups is in the order of $\cO(C^*(x))$, and the union bound goes through without incurring the cardinality of $\cY\backslash\cY_{1/(4C^*(x))}(x)$.

In detail, suppose that $\cY\backslash\cY_{1/(4C^*(x))}(x)=\{y_i\}_{i\ge1}$. We start with a single group $G_1=\varnothing$, and add $y_i$ to one of the groups sequentially. For each response $y_i\in\cY\backslash\cY_{1/(4C^*(x))}(x)$, if there exists group $G_j$ such that
\begin{align}\label{eq:group_condition}
\piref(y_i|x)+\sum_{y\in G_j}\piref(y|x)\le\frac1{2C^*(x)},
\end{align}
then we update $G_j$ with $G_j\cup\{a_i\}$ where $j$ is the smallest index that satisfies \eqref{eq:group_condition}. Otherwise, we create a new group $\{a_i\}$. From the construction of the groups, we can easily see that the probability of any group $G_j$ under the reference model satisfies
\begin{align}\label{eq:prob_group_ub}
\piref(G_i|x)=\sum_{a\in G_j}\piref(a|x)\le\frac1{2C^*(x)}.
\end{align}
Furthermore, the total number of groups $M$ should be no larger than $4C^*(x)$ because otherwise, suppose that \eqref{eq:group_condition} does not holds for $y_i$ and any existing group $ G_j(j\in[M])$ where $M>4C^*(x)-1$, i.e.,
\begin{align}\label{eq:prob_group_lb}
\sum_{y\in G_j}\piref(y|x)>\frac1{2C^*(x)}-\piref(y_i|x)>\frac1{4C^*(x)},
\end{align}
where the last inequality holds because $\piref(a)<1/(4C^*(x))$. We then have
\begin{align*}
1&=\sum_{y\in\cY}\piref(y|x)\\
&\ge\bigg[\piref(y_i|x)+\sum_{y\in G_1}\piref(y|x)\bigg]+\sum_{j=2}^M\bigg[\sum_{y\in G_j}\piref(y|x)\bigg]\\
&\ge\frac1{2C^*(x)}+(M-1)\cdot\frac1{4C^*(x)}\\
&>\frac1{2C^*(x)}+(4C^*(x)-1-1)\cdot\frac1{4C^*(x)}=1,
\end{align*}
where the first inequality holds because the union of $a_i$ and all existing groups is a subset of $\cA(x)$, the second inequality holds due to \eqref{eq:prob_group_lb}, and the last inequality holds due to the assumption of $M>4C^*(x)-1$. We have thus arrived at a contradiction, and we conclude that $M\le4C^*(x)$.

For each group, we apply the Chernoff bound:
\begin{align}
&\PP\bigg(\bigvee_{y\in G_j}\ind[\hat\pi(y)\ge3/(4C^*(x))]\bigg)\nonumber\\
&\le\PP\big(\hat\pi(G_j)\ge3/(4C^*(x))\big)\nonumber\\
&\le\exp\Big(-N\frac{(3/(4C^*(x))-\piref(G_i|x))^2}{3/(4C^*(x))+\piref(G_i|x)}\Big)\nonumber\\
&\le e^{-N/(20C^*(x))},\label{eq:group_high_prob}
\end{align}
where the first inequality holds because if the frequency of one response in $G_j$ is larger than $3/(4C^*(x))$, then the total frequency of group $G_j$ should be larger than $3/(4C^*(x))$; the second inequality holds due to the Chernoff bound; the last inequality holds due to \eqref{eq:prob_group_ub}. Applying the union bound to all groups,
\begin{align}
\PP\big(\hat\cY_{3/(4C^*(x))}\not\subset\cA_{1/(4C^*(x))}(x)\big)&=\PP\bigg(\bigvee_{y\in\cY\backslash\cY_{1/(4C^*(x))}}\ind[\hat\pi(y)\ge3/(4C^*(x))]\bigg)\nonumber\\
&\le\sum_{j=1}^M\PP\bigg(\bigvee_{y\in G_j}\ind[\hat\pi(y)\ge3/(4C^*(x))]\bigg)\nonumber\\
&\le Me^{-N/(20C^*(x))}\nonumber\\
&\le4C^*(x)e^{-N/(32C^*(x))},\label{eq:high_prob_2}
\end{align}
where the first inequality holds due to the union bound, the second inequality holds due to \eqref{eq:group_high_prob}, and the last inequality holds because $M\le4C^*(x)$ and $e^{-N/(20C^*(x))}\le e^{-N/(32C^*(x))}$. Combining \eqref{eq:high_prob_1} and \eqref{eq:high_prob_2}, using the union bound, we have
\begin{align*}
\PP(\cE)\ge1-5Ce^{-N/(32C^*(x))}.
\end{align*}
Thus, we have completed the proof of Lemma \ref{lemma:high_prob_event}.
\end{proof}

Using this lemma, we then proceed with the proof of Theorem \ref{theorem:upper_bound}:
\begin{proof}[Proof of Theorem \ref{theorem:upper_bound}]

Suppose that $\cE$ holds. If $y^*$ is included in the submitted responses, then the regret is $0$. We now consider the case where $y^*$ is not submitted. According to the definition of the coverage coefficient, we have
\begin{align*}
\piref(y^*|x)\ge\pi^*(y^*|x)/C^*(x)\ge1/C^*(x),
\end{align*}
so $y^*\in\cY_{1/C^*(x)}(x)$. Furthermore, since $\cY_{1/C^*(x)}(x)\subset\hat\cY_{3/(4C^*(x))}$ when $\cE$ holds, we have $y^*\in\hat\cY_{3/(4C^*(x))}$. Since $y^*$ is not selected as the output, we know that (i) at least $k$ responses are submitted because otherwise all responses in $\hat\cY_{3/(4C^*(x))}$ would be submitted, and (ii) $\hat r(x, y^**)\le\hat r(x, \tilde y_i)$ for any $i\in[k]$. We thus have
\begin{align}\label{eq:hat_r_lb}
\hat r(x, \tilde y_i)\ge\hat r(x, y^*)\ge r^*(x, y^*)-\epsilonopt(x),
\end{align}
where the second inequality holds due to Assumption \ref{assump: reward}. Therefore, the regret conditioned on event $\cE$ is
\begin{align}
\min_{i\in[k]}\{r^*(x, y^*)-r^*(x, \tilde y_i)\}&\le\epsilonopt(x)+\min_{i\in[k]}\{\hat r(x, \tilde y_i)-r_*(x, \tilde y_i)\}\nonumber\\
&\le\epsilonopt(x)+\sqrt{\frac{1}{k}\sum_{i=1}^k|\hat r(x, \tilde y_i)-r_*(x, \tilde y_i)|^2}\nonumber\\
&\le\epsilonopt(x)+\sqrt{\frac{4C^*(x)}k\sum_{i=1}^k\piref(\tilde y_i|x)|\hat r(x, \tilde y_i)-r^*(x, \tilde y_i)|^2}\nonumber\\
&\le\epsilonopt(x)+\sqrt{\frac{4C^*(x)}{k}\sum_{y\in\cY}\piref(y|x)|\hat r(x, y)-r^*(x, y)|^2}\nonumber\\
&=\epsilonopt(x)+\sqrt{\frac{4C^*(x)\epsilonRM^2(x)}{k}},\label{eq:regret_no_y*}
\end{align}
where the first inequality holds due to \eqref{eq:hat_r_lb}, the second inequality holds because the minimum is no larger than the average, the third inequality holds because $\piref(y|x)\ge1/(4C^*(x))$ for any $y\in\hat\cY_{3/(4C^*(x))}$ when $\hat\cY_{3/(4C^*(x))}\subset\cY_{1/(4C^*(x))}(x)$, the fourth inequality holds because $\{\tilde y_1, \dots, \tilde y_k\}$ is a subset of $\cY$, and the last equality holds due to the definition of the estimation error $\epsilonRM^2(x)$. Combining \eqref{eq:regret_no_y*} with the case where $y^*\in\{\tilde y_1, \dots, \tilde y_k\}$ and the regret is $0$, we conclude that under condition $\cE$,
\begin{align}\label{eq:regret_good}
r^*(x, y^*)-\max_{i\in[k]}r^*(x, \tilde y_i)\le\epsilonopt(x)+\sqrt{\frac{4C^*(x)\epsilonRM^2(x)}{k}}.
\end{align}
Finally, we take the complete expectation of the regret:
\begin{align*}
\mathrm{Regret}(x)&=\EE\Big[r^*(x, y^*)-\max_{i\in[k]}r^*(x, \tilde y_i)\Big|\cE\Big]\cdot\PP(\cE)+\EE\Big[r^*(x, y^*)-\max_{i\in[k]}r^*(x, \tilde y_i)\Big|\neg\cE\Big]\cdot\PP(\neg\cE)\\
&\le\bigg(\epsilonopt(x)+\sqrt{\frac{4C^*(x)\epsilonRM^2(x)}{k}}\bigg)\cdot\PP(\cE)+1\cdot\PP(\neg\cE)\\
&\le\epsilonopt(x)+\sqrt{\frac{4C^*(x)\epsilonRM^2(x)}{k}}+5C^*(x)e^{-N/(32C^*(x))},
\end{align*}
where the first inequality holds due to \eqref{eq:regret_good} and $\mathrm{Regret}(x)\le1$, and the second inequality holds because $\PP(\cE)\le1$ and due to Lemma \ref{lemma:high_prob_event}. Finally, when $N\ge 16 C^*(x) \log \big(kC^*(x)/\epsilonRM^2(x)\big)$, we have
\begin{align*}
    \mathrm{Regret}(x) \le \epsilonopt(x) + O\Big(\sqrt{{C^*(x)\epsilonRM^2(x)}/{k}}\Big).
\end{align*}
We complete the proof of Theorem \ref{theorem:upper_bound}.

\end{proof}

\section{Proof of Lower Bounds}
\label{sec:lowerbounds}
In this section, we will prove the lower bounds used in the main text of this paper. Specifically, we establish the results for majority voting (Theorem \ref{thm: lowerWM}), Best-of-$N$ (Theorem \ref{thm: lowerBoN}), and the general case of Pass@$k$ inference algorithms (Theorem \ref{theorem:lower_bound_general}). Before proceeding, we first establish an independent lower bound regarding $\epsilon_{\text{opt}}(x)$. This result is general and can be applied to any subsequent lower bound, introducing an additional $\epsilon_{\text{opt}}(x)$ term.

\subsection{Lower bound regarding $\epsilonopt(x)$}
We first study the following hard case where any algorithm for the Pass@$k$ inference problem suffers from the regret of $\Omega(\epsilonopt(x))$. Combining this lower bound with any algorithm-dependent lower bound $b$ (obtained from the analysis of a hard instance), we can show that the lower bound of the algorithm is
\begin{align*}
\Omega(\max\{\epsilonopt(x), b\})=\Omega(\epsilonopt(x)+b).
\end{align*}

\begin{lemma}
Assume that $\epsilonopt(x)\le\sqrt{C^*(x)\epsilonRM^2(x)}$ and $C^*(x)\ge2k$. Then there exists an instance $\cI=(\cX, \cY, \pi^*, r^*, \piref, \hat r)$ such that the coverage coefficient is $C^*(x)$, and $(r^*, \hat r)$ satisfy Assumptions \ref{assump:epsilon} and \ref{assump: reward}. Furthermore, for any prompt $x\in\cX$, the regret of any algorithm for the Pass@$k$ inference problem satisfies
\begin{align*}
\mathrm{Regret}(x)=\Omega(\epsilonopt(x)).
\end{align*}
\end{lemma}
\begin{proof}
For simplicity, we omit the prompt $x$ in our proof. We apply the idea of the averaging hammer, 
which considers a total of $M$ hard instances such that no algorithm can perform well on all instances. {It is a technique commonly used in the proof of lower bounds (see e.g., Theorem 24.1 in \citet{lattimore2020bandit}).} The responses set is $\{y_0, y_1, \dots, y_M\}$ for all $M$ hard instances. The reference policy and the approximate reward model are also shared by all instances:
\begin{gather*}
\piref(y_0)=1-M/C^*,\quad\piref(y_1)=\cdots\piref(y_M)=1/C^*;\\
\hat r(y_0)=0,\quad\hat r(y_1)=\cdots=\hat r(y_M)=1-\epsilonopt.
\end{gather*}
The hard instances are different only in the ground-truth reward model and $\pi^*$. For instance $\cI_j=(\cX, \cY, \pi_j^*, r_j^*, \hat r, \piref)$ where $j\in[M]$, we set
\begin{gather*}
\pi_j^*(y_i)=\delta_{ij},\quad
r_j^*(y_i)=\begin{cases}
0 & i=0;\\
1 & i=j;\\
1-\epsilonopt & \text{otherwise}.
\end{cases}
\end{gather*}
For all hard cases, the total estimation error is $\epsilonopt^2/C*\le\epsilonRM^2$. Among these $M$ hard instances, any algorithm that outputs up to $k$ responses will fail to output the optimal response in at least $M-k$ instances, inducing the regret of $\epsilonopt$. Therefore, the average regret of these $M$ instances is at least
\begin{align*}
\mathrm{Regret}\ge\frac{M-k}{M}\epsilonopt.
\end{align*}
Setting $M=2k$, we have $\mathrm{Regret}=\Omega(\epsilonopt)$.
\end{proof}

\subsection{Lower bound of majority voting (Theorem \ref{thm: lowerWM})}
\begin{proof}[Proof of Theorem \ref{thm: lowerWM}]

For simplicity, we omit the prompt $x$ in our proof. Consider the following hard instance. The size of the response set is $2 + k$, with $\cY = \{y_0, y^*, y_{1},y_2,\ldots,y_k\}$. The ground truth reward satisfies:
\begin{align*}
    r^*(y_0) = 0;\qquad r^*(y^*) = 1;\qquad r^*(y_{i}) = 1/2, \quad \forall 1\le i\le k.
\end{align*}
Therefore, the optimal policy $\pi^*$ satisfies:
\begin{align*}
    \pi^*(y_0) = 0;\qquad \pi^*(y^*) = 1;\qquad \pi^*(y_i) = 0,\quad.
\end{align*}
In this instance, we assume that the estimated reward function $\hat r$ is accurate. Let $\eta = 2w(1)/w(1/2)$. We further define the reference policy as:
\begin{align*}
    \piref(y_0) = 1- (1+\eta k)/C^*;\qquad \piref(y^*) = 1/C^*; \qquad\piref(y_{i}) = \eta / C^*, \quad \forall 1\le i\le k.
\end{align*}
The reference polity is well defined as long as $C^* \ge 1 + 2kw(1)/w(1/2)$. Now we consider the sampled responses $\hat y_1,\hat y_2,\ldots, \hat y_N$. Define
\begin{align*}
N^*= \sum_{j=1}^N \ind(\hat y_j = y^*);\qquad N_i= \sum_{j=1}^N \ind(\hat y_j =y_i),\quad\forall i\in[k].
\end{align*}
Then the expectations of $N^*$ and $N_i$ are
\begin{align*}
    \EE[N^*] = \frac{N}{C^*};\qquad \EE[N_i] = \frac{\eta N}{C^*}, \quad \forall 1 \le i \le k.
\end{align*}
Using the Chernoff bounds, we have
\begin{align}\label{eq:freq_separation}
    \PP\Big[\frac{N^*}{N} \ge  \frac{3}{2C^*}\Big] \le \exp \Big(\frac{-N}{9C^*}\Big),\quad\PP\Big[\frac{N_i}{N} \le  \frac{3\eta}{4C^*}\Big] \le \exp \Big(\frac{-N \eta}{4C^*}\Big).
\end{align}
Denote $\cE$ as the event such that
\begin{align*}
\frac{N^*}{N} \le  \frac{3}{2C^*};\qquad\frac{N_i}{N} \ge  \frac{3\eta}{4C^*},\quad\forall i\in[k].
\end{align*}
Taking the union bound with \eqref{eq:freq_separation}, we have
\begin{align*}
    \PP(\cE)\ge 1-\exp \Big(\frac{-N}{9C^*}\Big) - k \exp \Big(\frac{-N \eta}{4C^*}\Big)\ge1-(k+1) \exp \Big(\frac{-N}{9C^*}\Big),
\end{align*}
where the last inequality holds because $\eta > 1$. Under event $\cE$, we have
\begin{align*}
\frac{w(1/2)N_i}{w(1)N^*}=\frac{N_i/N}{N^*/N}\cdot\frac{w(1/2)}{w(1)}\ge\frac{3\eta/(4C^*)}{3/(2C^*)}\cdot\frac2\eta=1,
\end{align*}
where the inequality holds due to the definition of the event $\cE$ and the definition of $\eta$. Therefore, conditioned on event $\cE$, the (weighted) majority voting (Algorithm \ref{alg:WM}) will output $\{y_1,\ldots, y_k\}$ and suffer from a 1/2 regret. To summarize, the regret satisfies
\begin{align*}
    \text{Regret} \ge\PP(\cE)\cdot\EE[\mathrm{Regret}|\cE]\ge\frac{1}{2} \bigg(1 - (k+1) \exp \Big(\frac{-N}{9C^*}\Big)\bigg).
\end{align*}
When $N \ge 9C^*(x)\log(2k+2)$, 
\begin{align*}
    1 - (k+1) \exp \Big[\frac{-N}{9C^*}\Big] \ge 1/2.
\end{align*}
\end{proof}
\subsection{Lower bound of BoN (Theorem \ref{thm: lowerBoN})}

To prove Theorem \ref{thm: lowerBoN}, we construct two hard instances to accommodate two cases: (i) When $N$ is small, then it is very likely that $y^*$ does not even appear in $\{\hat y_1, \dots, \hat y_N\}$; (ii) When $N$ is large, then it is very likely that a number of responses that are suboptimal in $r^*$ but better than $y^*$ in $\hat r$ are sampled. The two hard instances share the same structure but are different in parameters.

\begin{proof}[Proof of Theorem \ref{thm: lowerBoN}]

For simplicity, we omit the prompt $x$. We consider two hard instances, one for $N\le C^*$ and the other for $N\ge C^*$.

\noindent\textbf{Case 1:} $N\le C^*$.
We consider a hard instance with $\cY=\{y_0, y^*\}$, and
\begin{gather*}
\pi^*(y_0)=0,\quad\pi^*(y^*)=1;\qquad r^*(y_0)=0,\quad r^*(y^*)=1;\\
\piref(y_0)=1-1/C^*,\quad\piref(y^*)=1/C^*;\qquad\hat r(y_0)=0,\quad\hat r(y^*)=1.
\end{gather*}
For this instance, the estimation errors are $\epsilonopt=\epsilonRM=0$. If no sample in $\hat y_1, \dots, \hat y_N$ is $y^*$, then the regret is $1$. The probability that $y^* \notin \{\hat y_1, \dots, \hat y_N\}$ is $(1-1/C^*)^N$. Therefore, we have
\begin{align*}
\mathrm{Regret}\ge(1-1/C^*)^N\ge(1-1/C^*)^{C^*}\ge1/4,
\end{align*}
where the second inequality holds because $N\le C^*$, and the second inequality holds because $C^*\ge2$. Therefore, the BoN algorithm incurs constant regret in this hard instance when $N\le C^*$.

\noindent\textbf{Case 2:} $N\ge C^*$. We consider the following hard instance: The response set is $\cY=\{y^*, y_0, y_1, \dots, y_M\}$. Let $p > 0$ be a parameter to be determined. The reward models are
\begin{gather*}
r^*(y^*)=1,\quad r^*(y_0)=0,\quad r^*(y_i)=1-\frac{\epsilonRM}{2\sqrt{p}};\\
\hat r(y^*)=1-\delta,\quad\hat r(y_0)=0,\quad \hat r(y_i)=1.
\end{gather*}
where $\delta < \epsilonopt$ is a sufficiently small positive number to ensure that the reward of $y_1, \dots, y_M$ is slightly better than $y^*$ in $\hat r$, but $y^*$ is still the optimal response in $r^*$. In this way, $\pi^*(y^*)=1$ and $\pi^*(y_i)=0$ for $i=0, 1, \dots, M$. The reference model satisfies
\begin{align*}
\piref(y^*)=1/C^*,\quad\piref(y_0)=1-1/C^*-p,\quad\piref(y_i)=p/M.
\end{align*}
For this instance, the coverage is $C^*$, and the estimation error is less than $\epsilonRM^2$ when $\delta$ is sufficiently small.

\textbf{Simple analysis.} We first consider a simple setting where $M=k$. When $\hat y_1, \dots, \hat y_N$ covers every response in $\{y_1, \dots, y_k\}$, then $\{y_1, \dots, y_k\}$ will be the output of BoN, causing the regret of $\epsilonRM/2\sqrt{p}$. The probability of any $y_i$ not being covered is
\begin{align*}
(1-p/k)^N.
\end{align*}
Using the union bound, the probability that there exists $y_i$ not being coverer is upper bounded by
\begin{align*}
    \PP[\exists i, y_i \notin \{\hat y_1,\ldots, \hat y_N\}] \le k(1-p/k)^N.
\end{align*}
Thus, the regret of making the wrong decisions in $y_1,\ldots, y_k$ is lower bounded by
\begin{align*}
    1-k(1-p/k)^N.
\end{align*}
Then the regret satisfies
\begin{align}
\label{eq:qw0001}
\mathrm{Regret}\ge\big(1-k(1-p/k)^N\big)\cdot\frac{\epsilonRM}{2 \sqrt p}.
\end{align}
In this instance, when $\sqrt{N\epsilonRM^2/[k\log(2k)]}/2  <1$, we select $p = (k/N)\cdot\log(2k)$. {Then, $\epsilonRM/(2\sqrt{p}) = \sqrt{N\epsilonRM^2/[k\log(2k)]}/2 < 1$. Thus, the constructed $r^*$ satisfies $0 \le r^*(\cdot) \le 1$.} Then we have
\begin{align*}
   1-k(1-p/k)^N &= 1- k\bigg(1-\frac{\log(2k)}{N}\bigg)^N \\
   & \ge 1- k\Big[\exp\Big(-\frac{\log(2k)}{N}\Big)\Big]^N\\
   &= 1- k\exp\big(-{\log(2k)}\big)\\
   &= 1/2, 
\end{align*}
{where the first inequality holds due to the basic inequality $1-x \le \exp(-x), \forall x\in \RR$. Substituting this into \eqref{eq:qw0001}, we have proved that} the regret can be lower bounded by $\Omega(\sqrt{N\epsilonRM^2/(k\log k)})$. Otherwise, {when $\sqrt{N\epsilonRM^2/[k\log(2k)]}/2  \ge 1$, let $p = \epsilonRM^2/4$. Then, the regret in \eqref{eq:qw0001} can be lower bounded by 
\begin{align*}
    \text{Regret} &\ge \bigg(1-k\Big(1-\frac{\epsilonRM^2}{4k}\Big)^N\bigg)\\
    & \ge \bigg(1-k\Big(1-\frac{\epsilonRM^2}{4k}\Big)^{4k\log(2k)/\epsilonRM^2}\bigg)\\
    & \ge \bigg(1-k\Big[\exp \Big(-\frac{\epsilonRM^2}{4k}\Big)\Big]^{4k\log(2k)/\epsilonRM^2}\bigg)\\
    &= 1- k\exp\big(-{\log(2k)}\\
    &= 1/2,
\end{align*}
where the third inequality holds due to the basic inequality $1-x \le \exp(-x), \forall x\in \RR$.
}
Therefore, we have
\begin{align*}
    \text{Regret} \ge \Omega\Big(\min\Big\{1, \sqrt{{N\epsilonRM^2}/{(k\log k)}}\Big\}\Big).
\end{align*}

This analysis will lead to an additional logarithmic term on $k$, which is unnecessary. To avoid this term, we consider the following improved analysis. 

\textbf{Improved analysis.} We consider the instance where $M=2k$.
Consider the event where at least $k$ responses among $y_1, \dots, y_M$ are covered by $\hat y_1, \dots, \hat y_N$. Since $\hat r(y_i)>\hat r(y^*)$ for $i=1, \dots, M$, the optimal responses $y^*$ is not included in $\tilde y_1, \dots, \tilde y_k$, which also incurs the regret of $\epsilonRM/(2\sqrt p)$. We now consider the probability of this event. Define the following random variables:
\begin{itemize}[leftmargin=*]
    \item Define $S$ as the number of samples within $y_1, \dots, y_M$, i.e.,
    \begin{align*}
    S=\sum_{i=1}^N\sum_{j=1}^M\ind[\hat y_i=y_j].
    \end{align*}
    \item Define $O_j$ as the occupancy of $y_j$, i.e.,
    \begin{align*}
    O_j=\bigvee_{i=1}^N\ind[\hat y_i=y_j].
    \end{align*}
    \item Define $D$ as the total occupancy of $\{y_1, \dots, y_M\}$, i.e.,
    \begin{align*}
    D=\sum_{j=1}^MO_j.
    \end{align*}
\end{itemize}
Our goal is to lower bound $\PP(D\ge k)$. Fix $s_0>k$. Using the total expectation formula, we have
\begin{align}
\PP(D\ge k)&=\sum_{s\ge k}\PP(D\ge k|S=s)\PP(S=s)\nonumber\\
&\ge\sum_{s\ge s_0}\PP(D\ge k|S=s)\PP(S=s)\nonumber\\
&\ge\PP(D\ge k|S=s_0)\PP(S\ge s_0),\label{eq:prob_D>k_decomp}
\end{align}
where the first inequality holds because $s_0\ge k$, and the second inequality holds because $\PP(D\ge k|S=s)\ge\PP(D\ge k|S=s_0)$ when $s\ge s_0$. We then calculate the two probabilities separately. We first use the Chernoff bound to characterize $\PP(S\ge s_0)$. The expectation of $S$ is
\begin{align*}
\EE[S]=\sum_{i=1}^N\PP(\hat y_i\in\{y_1, \dots, y_M\})=Np.
\end{align*}
Then by the Chernoff bound, we have
\begin{align}\label{eq:prob_S>s_0}
\PP(S\ge s_0)\ge1-\exp\Big(-\frac{(Np-s_0)^2}{2Np}\Big).
\end{align}
We then calculate the conditional probability $\PP(D\ge k|S=s_0)$, and we assume without loss of generality that $\hat y_1, \dots, \hat y_{s_0}$ fall within $\{y_1, \dots, y_M\}$. Conditioned on this event $\cE$, we have $\PP(\hat y_i=y_j)=1/M$ for $1\le i\le s_0$ and $1\le j\le M$. Although we cannot use the vanilla Chernoff bound to bound $\PP(D\ge k|S=s)$, we can use the Chernoff bound for \textbf{negatively-correlated} random variables {\citep{dubhashi1996balls} (or See Theorem 4.3 in \citet{dubhashi1998concentration})} to bound the probability. We first calculate the expectation of $D$, which is
\begin{align*}
\EE[D|S=s_0]=M\EE[O_j]=M(1-\PP[\hat y_i\neq y_j, \forall i\in[s_0]])=M(1-(1-1/M)^{s_0}).
\end{align*}
We then verify that $O_1, \dots, O_M$ are negatively correlated, which is to show that for any subset $\cJ\subset[M]$, we have $\EE[\prod_{j\in\cJ}O_j]\le\prod_{j\in\cJ}\EE[O_j]$, i.e., $\PP(O_j=1, \forall j\in\cJ)\le\prod_{j\in\cJ}\PP(O_j=1)$. We prove by induction with respect to the cardinality of $\cJ$. The inequality is trivial When $|\cJ|=1$. Suppose that the inequality holds for all $\cJ$ such that $|\cJ|\le n$. It then suffices to show the inequality holds for $\cJ=[n+1]$. Note that
\begin{align*}
&\PP(O_1=1, \dots, O_{n+1}=1)\\
&=\PP(O_1=1, \dots, O_n=1)-\PP(O_1=1, \dots, O_n=1|O_{n+1}=0)\cdot\PP(O_{n+1}=0)\\
&=\PP(O_1=1, \dots, O_n=1)\cdot\PP(O_{n+1}=1)\\
&\quad+\big[\PP(O_n=1, \dots, O_n=1)-\PP(O_1=1, \dots, O_n=1|O_{n+1}=0)\big]\cdot\PP(O_{n+1}=0),
\end{align*}
Using the induction hypothesis, we have
\begin{align*}
\PP(O_1=1, \dots, O_n=1)\cdot\PP(O_{n+1}=1)\le\prod_{j=1}^{n+1}\PP(O_j=1).
\end{align*}
It then suffices to show that
\begin{align*}
\PP(O_n=1, \dots, O_n=1)\le\PP(O_1=1, \dots, O_n=1|O_{n+1}=0),
\end{align*}
which is trivial because the event $\hat y_i=y_j(j\in[n])$ becomes more likely conditioned of the event that $\hat y_i\neq y_{n+1}$. Therefore, the inequality holds for $|\cJ|=n+1$, and we complete the verification of $O_j$ being negatively correlated. Therefore, using the Chernoff bound for negatively-correlated random variables, we have
\begin{align}\label{eq:chernoff_neg_corr}
\PP(D\ge k|S=s_0)\ge1-\exp\bigg(-\frac{\{M[1-(1-1/M)^{s_0}]-k\}^2}{2M[1-(1-1/M)^{s_0}]}\bigg).
\end{align}
Substituting \eqref{eq:prob_S>s_0} and \eqref{eq:chernoff_neg_corr} into \eqref{eq:prob_D>k_decomp}, we have
\begin{align}
\notag
&\mathrm{Regret}\ge\PP(D\ge k)\cdot\frac{\epsilonRM}{2\sqrt p}\\\label{eq:qw0002}
&\ge\frac{\epsilonRM}{2\sqrt p}\cdot\bigg[1-\exp\bigg(-\frac{\{M[1-(1-1/M)^{s_0}]-k\}^2}{2M[1-(1-1/M)^{s_0}]}\bigg)\bigg]\cdot\bigg[1-\exp\bigg(-\frac{(Np-s_0)^2}{2Np}\bigg)\bigg].
\end{align}
Let $M=2k,s_0=3k$. If $\sqrt{N\epsilonRM^2/k}/4 \le 1$, we set $p=4k/N$. {Then, $\epsilonRM/(2\sqrt{p}) = \sqrt{N\epsilonRM^2/k}/4 < 1$. Thus, the constructed $r^*$ satisfies $0 \le r^*(\cdot) \le 1$.}
In this case, we have
\begin{align*}
1-(1-1/M)^{s_0}=1-\Big(1-\frac1{2k}\Big)^{3k}\ge 1-e^{-1.5}\ge\frac34.
\end{align*}
We thus have
\begin{align*}
&1-\exp\bigg(-\frac{\{M[1-(1-1/M)^{s_0}]-k\}^2}{2M[1-(1-1/M)^{s_0}]}\bigg)\\
&\ge1-\exp\Big(-\frac{(2k\cdot3/4-k)^2}{2\cdot2k\cdot3/4}\Big)\\
&=1-e^{-k/12}\ge1-e^{-1/12},
\end{align*}
where the second inequality holds because $k\ge1$. We also have $Np=4k$, so
\begin{align*}
&1-\exp\bigg(-\frac{(Np-s_0)^2}{2Np}\bigg)=1-\exp\bigg(-\frac{(4k-3k)^2}{2\cdot4k}\bigg)=1-e^{-k/8}\ge1-e^{8},
\end{align*}
where the last inequality holds because $k\ge1$. Combining all the above, we have
\begin{align*}
\mathrm{Regret}\ge\frac{\epsilonRM}{\sqrt{4k/N}}\cdot(1-e^{-1/12})\cdot(1-e^{-1/8})\ge0.004\sqrt{\frac{N\epsilonRM^2}{k}}.
\end{align*}
Otherwise, {if $\sqrt{N\epsilonRM^2/k}/4 \ge 1$ let $p = \epsilonRM^2/4$. Then, with the same argument as that in the {Simple analysis} part, the regret is lower bounded by $\Omega(1)$.} 
Therefore, we have
\begin{align*}
    \text{Regret} \ge \Omega\Big(\min\Big\{1, {\sqrt{N\epsilonRM^2/{k}}}\Big\}\Big).
\end{align*}
\end{proof}

\subsection{General lower bound (Theorem \ref{theorem:lower_bound_general})}

We first provide a more general version of Theorem \ref{theorem:lower_bound_general}:
\begin{theorem}\label{theorem:lower_bound_general_appendix}
Assume that $C^*(x)\ge\max\{k, 2\}$. Then for any positive integer $M\in[k, C^*(x)]$ and any algorithm $A$ that outputs $k$ responses, there exists a hard instance $\cI=(\cX, \cY, \pi^*, r^*, \piref, \hat r)$ such that the coverage is $C$, the estimation error is $\epsilonRM^2$, and the regret of algorithm $A$ satisfies
\begin{align*}
\mathrm{Regret}(x)\ge\frac{M-k}{M}\sqrt{\frac{C^*(x)\epsilonRM^2}{M-1}}.
\end{align*}
\end{theorem}

When $C\ge 2k$, we can set $M=2k$ and obtain the regret lower bound of $\Omega(\sqrt{C\epsilonRM^2/k})$ in Theorem \ref{theorem:lower_bound_general}. We now present the proof of Theorem \ref{theorem:lower_bound_general_appendix}.

\begin{proof}[Proof of Theorem \ref{theorem:lower_bound_general_appendix}]

We consider the case of $\cX=\{x\}$, and omit the prompt $x$ in $A(x)$, $\piref(\cdot|x)$, $\hat r(x, \cdot)$, etc.
 
To prove Theorem \ref{theorem:lower_bound_general}, we again apply the idea of the averaging hammer, and consider a total of $M$ hard instances such that no algorithm can perform well on all instances. All of these hard instances have a total of $M+1$ possible responses $\cY=\{y_0, \dots, y_M\}$, and we aim to make $y_1, \dots, y_M$ hard to distinguish from each other. In detail, all hard instances also share the same reference model and the same $\hat r$:
\begin{gather*}
\piref(y_0)=1-M/C,\quad\piref(y_1)=\cdots=\piref(y_M)=1/C;\\
\hat r(y_0)=0,\quad \hat r(y_1)=\cdots=\hat r(y_M)=1.
\end{gather*}
For hard instance $\cI_j(j\in[M])$, we make $y_j$ the optimal response with ground truth reward being $1$ and $\pi^*(y_j)=1$, and make all other responses suboptimal with a gap of $\delta$, i.e., $\cI_j=(\cX, \cY, \pi_j^*, r_j^*, \piref, \hat r)$, where
\begin{align*}
\pi_j^*(y_l)=\delta_{jl},\quad
r_j^*(y_l)=\begin{cases}
0 & l=0;\\
1 & l=j;\\
1-\delta & \text{otherwise}.
\end{cases}
\end{align*}
In this hard instance, the coverage is $C$, and in order to make the estimation error equal to $\epsilonRM^2$, we require
\begin{align*}
(M-1)\cdot\delta^2\cdot1/C=\epsilonRM^2,
\end{align*}
which indicates that $\delta=\sqrt{C\epsilonRM^2/(M-1)}$. Since any algorithm can only output a maximum of $k$ different responses, it cannot output the optimal response in at least $M-k$ out of the $M$ hard instances, suffering from the regret of at least $\delta$. Therefore, the averaged regret of the $M$ instances is at least
\begin{align*}
\frac1M\sum_{j=1}^M\EE_{\tilde y_1, \dots, \tilde y_k\sim A}\big[r_j^*(y_j)-\max\big\{r_j^*(\tilde y_1), \cdots, r_j^*(\tilde y_k)\big\}\big]\ge\frac1M\cdot(M-k)\cdot\delta=\frac{M-k}{M}\sqrt{\frac{C\epsilonRM^2}{M-1}}.
\end{align*}
Therefore, there exists an instance $\cI_{j^*}$ within the $M$ hard instances such that
\begin{align*}
\EE_{\tilde y_1, \dots, \tilde y_k\sim A}\big[r_{j^*}^*(y_{j^*})-\max\big\{r_{j^*}^*(\tilde y_1), \cdots, r_{j^*}^*(\tilde y_k)\big\}\big]\ge\frac{M-k}{M}\sqrt{\frac{C\epsilonRM^2}{M-1}}.
\end{align*}
\end{proof}
{
\section{Further Ablation Studies}
\label{sec:further}
In this section, we present additional ablation studies on the choice of $\alpha$, using the AIME24 and GSM8K datasets.
}
\begin{table}[htbp]
\centering
\scriptsize
\vspace{8pt}
\caption{{Ablation study of $\alpha$ of BoM in AIME24, Qwen3-4B}}
\label{tab:ablation1}
\begingroup
\resizebox{0.6\textwidth}{!}{%
\begin{tabular}{c|c|c|c|c|c}
\toprule
Pass@$k$&1 & 2 &3 & 5& 10 \\
\midrule
$\alpha=0$ (BoN) 
    &  53.3     & 60 & 63.3& 70& 73.3 \\
$\alpha=0.003$ 
    &  56.7     & 66.7 & 70& 73.3& 73.3 \\
$\alpha=0.005$&    53.3   & 66.7  &70 &70 &70 \\
  $\alpha=0.007$  &    56.7   & 70 & 70& 70& 70\\
  $\alpha=0.011$  &    50   & 63.3 & 63.3& 63.3& 63.3\\
  $\alpha=0.015$  &    50   & 60 & 60& 60& 60\\
  Majority voting  &    36.7   & 43.3 & 46.7&46.7& 60\\
\bottomrule
\end{tabular}%
}
\endgroup
\end{table}
\begin{table}[htbp!]
\centering
\scriptsize
\vspace{8pt}
\caption{{Ablation study of $\alpha$ of BoM in GSM8k, Qwen3-4B}}
\label{tab:ablation2}
\begingroup
\resizebox{0.6\textwidth}{!}{%
\begin{tabular}{c|c|c|c|c|c}
\toprule
Pass@$k$&1 & 2 &3 & 5& 10 \\
\midrule
$\alpha=0$ (BoN) 
    &  83.58     & 88.56 & 91.50& 92.38& 93.26 \\
$\alpha=0.003$ 
    &  85.34     & 90.32 & 92.38& 92.67& 92.67 \\
$\alpha=0.005$&    85.34   & 90.03  &92.08 &92.38 &92.38 \\
  $\alpha=0.007$  &    85.92   & 90.62 & 92.08& 92.08& 92.08\\
  $\alpha=0.011$  &    86.22   & 90.32 & 91.50& 91.50& 91.50\\
  $\alpha=0.015$  &    86.22   & 90.32 & 91.20& 91.20& 91.20\\
  Majority voting  &    82.70   & 90.03 & 91.20&92.96& 92.96\\
\bottomrule
\end{tabular}%
}
\endgroup
\end{table}
{
\section{Ablation study of reward model}
To further demonstrate that our findings are not tied to any particular reward model, we additionally evaluate using InternLM2-reward \citep{cai2024internlm2}. This model is not specifically trained for mathematical reasoning and shows weaker performance in the evaluation of \citet{acemath2024}. Even under this weaker reward model, our algorithm behaves consistently and leads to the same overall conclusion.
\begin{figure*}[!h]
\centering   
\subfigure[AIME24]{\includegraphics[width=0.80\textwidth]{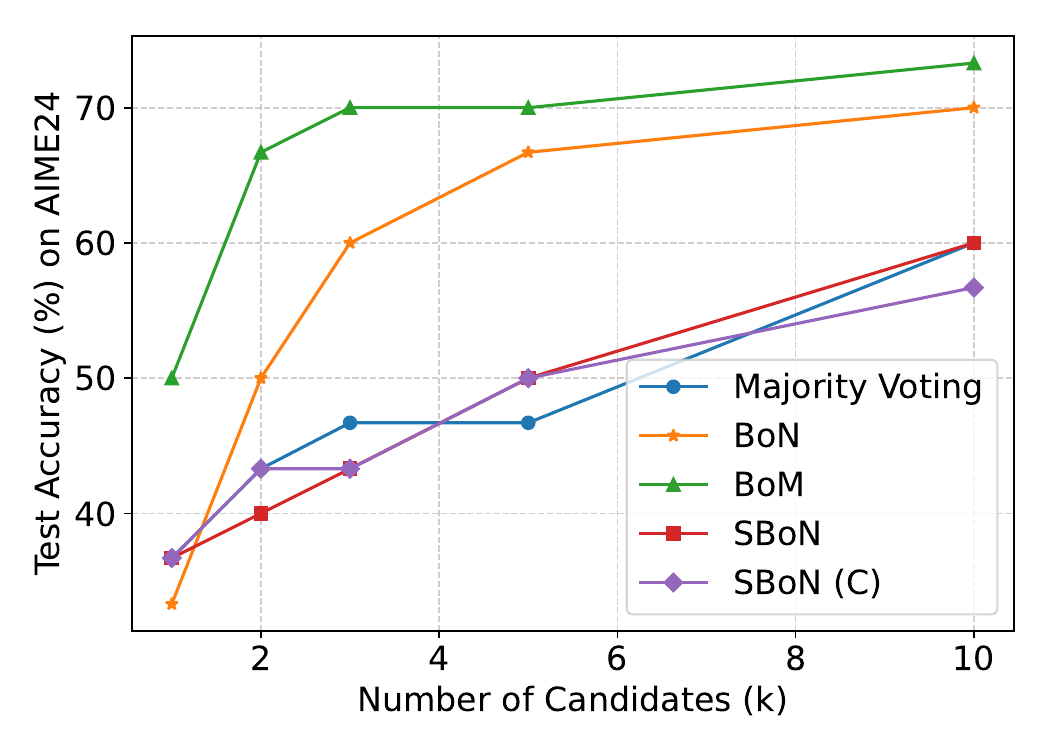}}
\caption{The results of different $k$ with $N=500$ using  InternLM2-reward, Qwen3-4B-Instruct .}
\label{fig:interLM-aime}
\end{figure*}}
\section{Additional Experiments}\label{app:addtional-exp}

\begin{figure*}[!bht]
\centering   
\subfigure[GSM8K]{\label{fig:gsm8k-qwen2.5-1.5b}\includegraphics[width=0.32\textwidth]{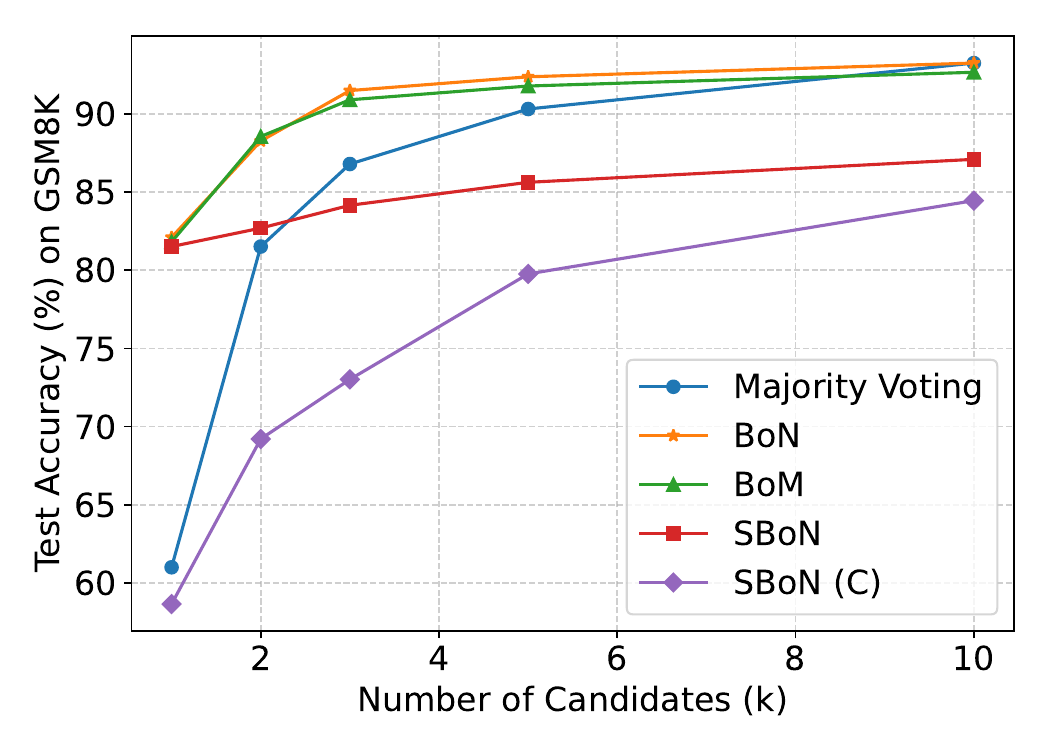}}
\subfigure[MATH-500]{\label{fig:math500-qwen2.5-1.5b}\includegraphics[width=0.32\textwidth]{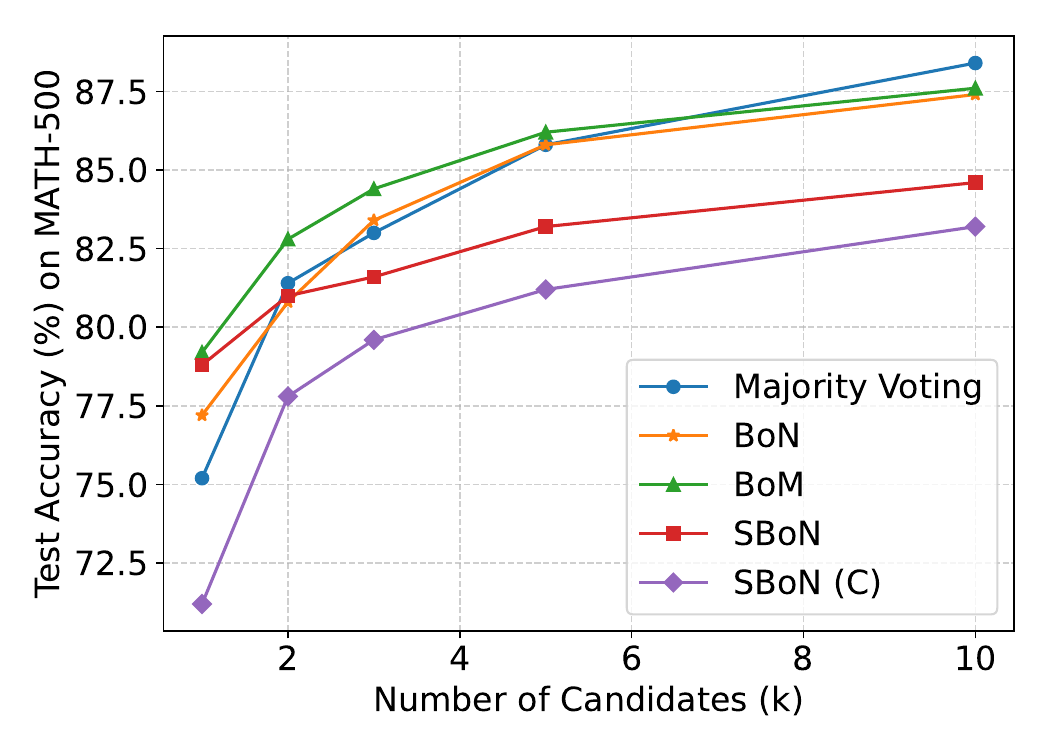}}
\subfigure[AIME24]{\label{fig:aime24-qwen2.5-1.5b}\includegraphics[width=0.32\textwidth]{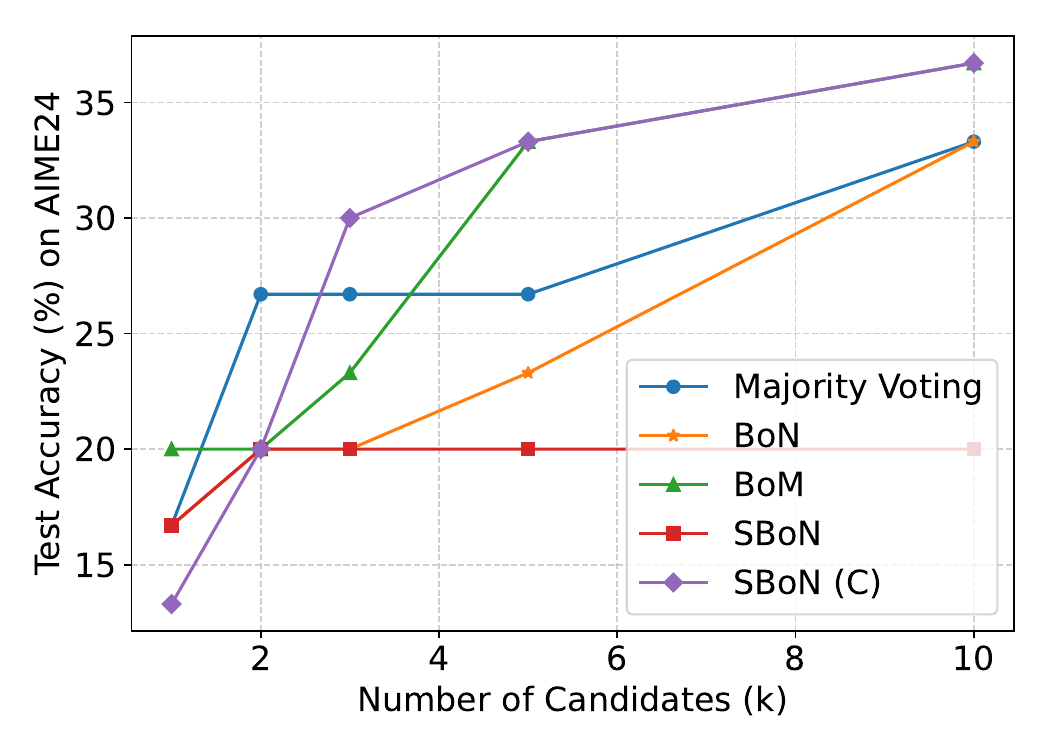}}
\caption{The results of different $k$ with $N=500$ on Qwen2.5-1.5B.}
\label{fig:scaling-k-qwen2.5-1.5b}
\end{figure*}

In this section, we conduct experiments on an additional model, Qwen2.5-Math-1.5B-Instruct (Qwen2.5-1.5B) for more results. The other experiment setups follows the experiments on Qwen3 unless specified. The results on Qwen2.5-1.5B are compiled in Figure~\ref{fig:scaling-k-qwen2.5-1.5b}. In particular, BoM matches the performance of BoN on GSM8k and outperforms BoN on MATH-500 and AIME24. The performance of BoM also surpasses majority voting on GSM8k and MATH-500 with $k \le 5$. These results shows that BoM demonstrates a better overall performance over baselines when $k$ is small.


\section*{The Use of Large Language Models (LLMs)}
We use LLMs as a tool to refine our writing and correct grammatical errors.
\end{document}